\documentclass[11pt]{article}

\usepackage[a4paper,margin=2cm]{geometry}
\usepackage{amsmath,amssymb,mathtools,amsfonts}
\usepackage{amsthm}
\usepackage{bm}
\usepackage{graphicx}
\usepackage[font=footnotesize,labelfont=bf]{caption}
\usepackage{subcaption}
\usepackage{xcolor}
\newcommand{\Var}{\mathrm{Var}}
\newcommand{\dd}{\mathrm{d}}
\newcommand{\tr}{\mathrm{tr}}
\usepackage{hyperref}
\usepackage{pgfplots}
\pgfplotsset{compat=1.18}
\usepgfplotslibrary{colormaps}
\usepackage{algorithm}
\usepackage{algpseudocode}
\usepackage{siunitx}
\usepackage{enumitem}
\usepackage{pgfplotstable}
\usepackage{booktabs}

\usepackage[utf8]{inputenc}
\usepackage{textcomp}
\DeclareUnicodeCharacter{2212}{-}

\hypersetup{
  colorlinks=true,
  linkcolor=blue!60!black,
  citecolor=purple!60!black,
  urlcolor=blue!60!black
}
\usepackage[numbers,compress,sort]{natbib}
\numberwithin{equation}{section}
\theoremstyle{plain}
\newtheorem{theorem}{Theorem}[section]
\newtheorem{proposition}[theorem]{Proposition}
\newtheorem{lemma}[theorem]{Lemma}

\theoremstyle{definition}

\theoremstyle{definition}

\newtheorem{assumption}[theorem]{Assumption}
\newcommand{\vect}[1]{\bm{#1}}

\newcommand{\E}{\mathbb{E}}
\newcommand{\R}{\mathbb{R}}
\newcommand{\Normal}{\mathcal{N}}

\DeclareMathOperator{\softmax}{softmax}

\title{Out-of-Distribution Detection in Molecular Complexes via Diffusion Models for Irregular Graphs}

\date{}

\begin{document}
\maketitle

\medskip
David Graber\textsuperscript{1,2}*, 
Victor Armegioiu\textsuperscript{1}*, 
Rebecca Buller\textsuperscript{3}, 
Siddhartha Mishra\textsuperscript{1}

\smallskip \scriptsize
$*$ These authors contributed equally to this work and share first authorship. Author order was determined by coin toss.

\bigskip 

\scriptsize
1 Seminar for Applied Mathematics, Department of Mathematics and ETH AI Center, ETH Zurich, 8092 Zurich, Switzerland\\ 
2 Institute for Computational Life Sciences, Zurich University of Applied Sciences, 8820 Waedenswil, Switzerland\\
3 Department of Chemistry, Biochemistry and Pharmaceutical Sciences, University of Bern, 3012 Bern, Switzerland

\bigskip
\textbf{Correspondence to:}\\ 
David Graber (\texttt{david.graber@sam.math.ethz.ch})\\
Victor Armegioiu (\texttt{victor.armegioiu@sam.math.ethz.ch})\\
Prof. Dr. Rebecca Buller (\texttt{rebecca.buller@unibe.ch})\\
Prof. Dr. Siddhartha Mishra (\texttt{siddhartha.mishra@ethz.ch})\\\\

\normalsize
\bigskip
\begin{abstract}
\end{abstract}

\noindent Predictive machine learning models generally excel on in-distribution data, but their performance degrades on out-of-distribution (OOD) inputs. Reliable deployment of machine learning models therefore requires robust OOD detection, yet identifying OOD data in irregular 3D graphs remains a fundamental challenge. We present the first unsupervised OOD detection framework specifically designed for 3D geometric graphs, using a generative diffusion model to approximate the joint distribution of geometry and discrete features. We demonstrate that probability-flow ODE trajectories provide a quantitative, per-complex notion of typicality with respect to the training distribution. By augmenting likelihoods with eighteen trajectory-level geometric features we effectively bypass the complexity bias inherent in standard density estimation. Our framework accurately identifies OOD protein-ligand complexes and anticipates the error profiles of independent binding-affinity predictors without requiring task-specific labels. In this sense, this OOD quantification can act as a certificate for benchmark datasets and the connected generalization claims.

\newpage
\tableofcontents

\begin{figure}[p]
    \centering
    \includegraphics[width=\textwidth]{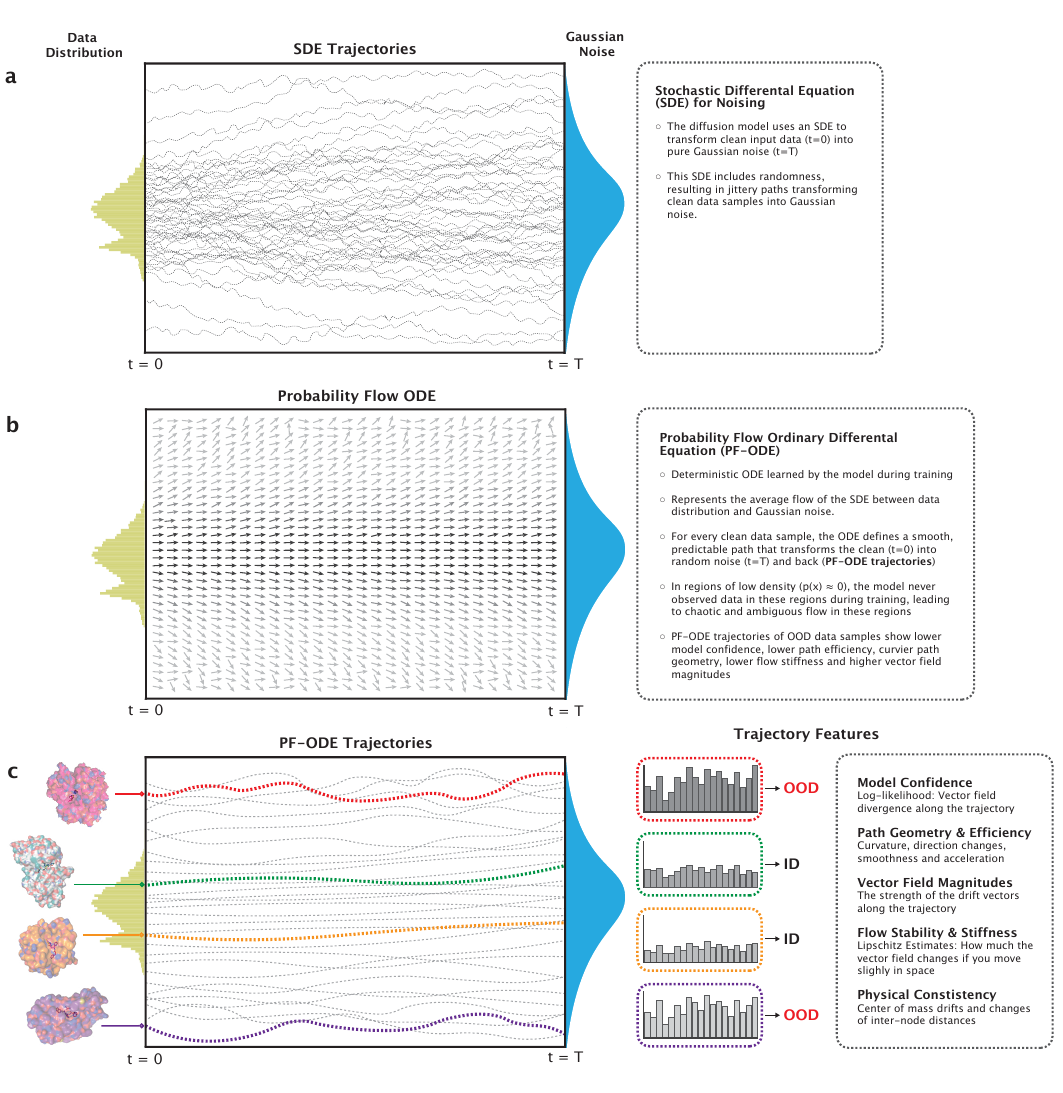}
    \caption{
\textbf{Schematic of out-of-distribution (OOD) detection via PF-ODE trajectory analysis:} \textbf{a)} Stochastic Differential Equation (SDE): The diffusion framework defines a forward process that transforms the input data distribution (protein-ligand complexes) into a Gaussian noise distribution via stochastic, noisy paths. \textbf{b)} Probability Flow ODE (PF-ODE): The trained model learns a deterministic vector field representing the average marginal flow of the SDE. This defines a continuous, reversible mapping between the noise and data spaces, transporting every sample along a unique, deterministic path (trajectory). \textbf{c)} Comparison of trajectories for in-distribution (ID) versus OOD samples: ID samples follow efficient, direct paths through regions where the vector field is well-constrained by training data. In contrast, OOD samples traverse low-density regions where the learned flow is ambiguous and chaotic due to a lack of training samples, resulting in erratic, high-tortuosity paths. Geometric features extracted from these trajectories, such as path efficiency, flow energy, and Lipschitz stability, serve as robust signals to detect OOD inputs.
    }
    \label{fig:concept}
\end{figure}

\newpage
\section{Introduction}

Out-of-distribution (OOD) detection is essential in machine learning, because real-world data often does not match the data distribution used for model training. In these cases, models often make overly confident predictions on unfamiliar or shifted inputs, which can lead to serious errors, especially in high-stakes domains like diagnostics or drug discovery. Effective OOD detection provides a safeguard by identifying when a model is operating outside its reliable domain, allowing users to reject predictions or handle uncertain cases more appropriately. This not only improves safety but also builds trust in machine learning models deployed in dynamic environments. 

\smallskip Traditional OOD detection approaches can be broadly categorized into supervised discriminative methods and unsupervised generative methods (see Section \ref{sec:related_work}). Discriminative approaches typically leverage scores derived from pretrained classifier outputs \cite{Hendrycks2017, Liang2017, Liu2020, Sun2021}. However, their dependence on in-distribution labels limits their applicability in unsupervised settings where such labels are unavailable. Consequently, unsupervised approaches typically employ generative models and fall into two primary sub-categories: reconstruction-based and likelihood-based methods. 

Reconstruction-based approaches \cite{AnCho2015, Yang2021, Ataeiasad2024} operate on the assumption that a generative model trained on in-distribution data will fail to accurately reconstruct OOD samples. However, these methods frequently suffer from reconstruction shortcuts, where powerful decoders successfully reconstruct simple OOD inputs despite their semantic novelty, rendering the reconstruction error metric unreliable \cite{Denouden2018, Zong2018}.

Alternatively, likelihood-based methods explicitly model the probability density of the training data, flagging test samples in low-density regions as OOD. These approaches frequently employ denoising diffusion probabilistic models (DDPMs) to capture this density. By learning the probability density function of the training distribution, the DDPM can output a scalar likelihood score indicating the degree to which a new sample matches the learned distribution \cite{Song2020, Goodier2023}. While theoretically sound, these models often struggle with the complexity bias \cite{Nalisnick2018, Serra2019, Kirichenko2020}, where low-complexity OOD data is assigned high likelihoods, leading to false negatives.
To mitigate these limitations, Heng et al. (2024) \cite{Heng2024} demonstrated that the rate-of-change and curvature of diffusion trajectories can serve as effective OOD signals for image benchmarks. Their trajectory-focused approach proved resistant to the complexity bias that often confounds likelihood-based OOD classification. 

Crucially, while OOD detection has been adapted for graph data \cite{Wu2023, Ma2021, Li2022, Shen2024}, existing methods predominantly focus on 2D topological graphs \cite{Ma2021, Li2022, Shen2024}. These approaches neglect the precise 3D geometry that is essential for modeling 3-dimensional data such as molecular complexes, leaving a critical gap in the growing field of geometric deep learning.

\smallskip Here, we present a novel OOD detection framework built upon a diffusion model tailored for 3D irregular graphs. By modeling likelihood scores jointly with an expanded set of 18 distinct diffusion trajectory features, we derive complementary, robust OOD detection signals. This strategy effectively bypasses the complexity bias inherent in isolated log-likelihood evaluation, enabling robust OOD detection within the 3D graph domain. We showcase this OOD detection approach using a database of protein-ligand interaction complexes called PDBbind \cite{Wang2004, Liu2015}, which are typically modeled as 3-dimensional geometric graphs combining discrete attributes (atom and residue types) with continuous spatial coordinates. Despite the complexity of this domain, our framework accurately identifies OOD complexes characterized by protein pockets unseen during training.

\bigskip To train a diffusion model on this complex data that combines 3D geometry and discrete graph features, we introduce a unified continuous diffusion framework that unifies both 3D geometry and discrete features within a single Euclidean state. By mapping the categorical data onto a continuous, spherical embedding space and concatenating these chemical embeddings with the 3D spatial coordinates, we can subject the entire molecular state to a single diffusion process. A single SE(3)-equivariant graph neural network (EGNN) consumes this joint state of coordinates and noisy categorical embeddings and outputs per-node clean coordinates and probability (logits) of each chemical type. We then calculate a posterior mean (essentially a weighted average of the learned chemical prototypes based on the predicted probabilities) and compute the score analytically from this "clean" signal. This method of posterior-mean interpolation allows us to train the model using standard cross-entropy loss while maintaining the smooth, continuous dynamics required for the diffusion process. 

\medskip To detect OOD samples using this diffusion model trained on in-distribution protein-ligand complexes, we make use of the specific properties of DDPMs: These generative models define a continuous forward process that gradually transforms a clean input structure into pure Gaussian noise using a stochastic differential equation (SDE, Figure \ref{fig:concept}a). Then, they learn a corresponding reverse process to travel back from noise to the original clean structure by solving the backward SDE.
This admits a mathematically equivalent formulation of the diffusion process known as the probability-flow ordinary differential equation (PF-ODE) \cite{Song2020}. This ODE defines a deterministic vector field representing the average marginal flow of the SDE, establishing a continuous, reversible mapping between the noise and data spaces (Figure \ref{fig:concept}b). This mapping transports every sample along a unique, deterministic path, which we refer to as its PF-ODE trajectory (Figure \ref{fig:concept}c). 

In this setting, we can compute the exact log-likelihood of any sample by integrating the divergence of the drift term along its PF-ODE trajectory. While this score theoretically estimates how OOD a sample is, it has been repeatedly shown to suffer from complexity bias, where very simple inputs are often assigned in-distribution likelihoods despite being OOD \cite{Nalisnick2018, Serra2019, Kirichenko2020}. To mitigate this, we augment the likelihood scores with additional features extracted from the sample's PF-ODE trajectories (Figure \ref{fig:concept}c). The underlying idea is that in-distribution (ID) and OOD samples are distinguishable by the geometric characteristics of their PF-ODE trajectories: ID samples follow efficient, direct paths through regions where the vector field is well-constrained by training data. In contrast, OOD samples traverse low-density regions where the learned flow is ill-defined due to a lack of similar training samples, resulting in erratic and chaotic trajectories. Features extracted from these trajectories, such as path efficiency, flow energy, and Lipschitz stability, serve as robust signals to detect OOD inputs.

\medskip To demonstrate the efficacy of this trajectory-focused approach in the 3D graph domain, we constructed a custom train-test split of the PDBbind database. We formed seven strict OOD test sets by systematically excluding entire protein families from the training data. This rigorous exclusion ensures that these proteins are never encountered during training and appear novel and atypical to the model at inference. Additionally, we incorporated the CASF2016 benchmark set and a standard validation set to serve as intermediate level OOD test sets. 

\smallskip After training the diffusion model on this specific split, we quantified the OOD signal inherent in PF-ODE trajectories and tested its ability to flag the OOD test sets as anomalies. Our analysis demonstrates that the derived log-likelihoods contain considerable OOD signal and are chemically meaningful. Across eight out of nine datasets, they correctly rank held out protein families as increasingly OOD, aligning with bioinformatically computed similarities. These likelihoods even correlate with the prediction errors of an independent binding affinity prediction model that was trained on the same data. However, despite this alignment, the log-likelihood distributions retained significant overlap with the training data, making robust separation of OOD and ID samples difficult. Crucially, for one dataset comprising low complexity interactions, the likelihoods wrongly suggested ID status, which indicates a severe complexity bias. Concluding that log-likelihoods alone are insufficient for reliable OOD detection, we incorporated additional PF-ODE trajectory statistics. This augmented feature space yields a substantially stronger OOD detector that consistently distinguishes ID from OOD complexes, successfully classifying the low complexity dataset that was previously misidentified.

\bigskip Overall, our results indicate that diffusion models can serve as practical OOD detectors. They can be integrated with deployed predictive machine learning models to provide a supporting confidence score and flag unreliable predictions. Crucially, we show that augmenting log-likelihoods with geometric features of the diffusion trajectories significantly increases detection accuracy and provides immunity against the complexity bias. While validated here within the challenging 3D graph domain, this OOD detection framework is inherently generalizable to any data modality amenable to continuous diffusion and requires no labels.

\clearpage
\section{Results}

To rigorously evaluate the diffusion model's OOD detection capability, we used a custom data split of the PDBbind (v.2020) dataset \cite{kopko2025generalization} with nine test datasets designed to simulate a spectrum of distribution shifts relative to the training data.

\begin{itemize}[itemsep=0.8pt]
    \item \textbf{Minimal shift (validation set):} A standard validation set drawn from the training distribution.
    
    \item \textbf{Intermediate shift (CASF2016):} We applied CleanSplit filtering \cite{GEMS} to the training data, removing structural overlaps with the CASF-2016 benchmark.
    
    \item \textbf{Strong shift (isolated protein families):} We constructed seven strict OOD sets by isolating entire protein families from the training data, including serine/threonine-protein kinases, estrogen receptors, HIV proteases, $\alpha$-carbonic anhydrases, urokinase-type plasminogen activators, HSP82 chaperones, and a cluster of transporter proteins (cluster centers 1nvq, 2p15, 3o9i, 3dd0, 1sqa, 2vw5, 3f3e; see Section \ref{met:Datasets}).
\end{itemize}

In the following sections, we first validate the OOD level of our test sets through an extensive bioinformatic similarity analysis, and then analyze the power of the diffusion model to detect them. We evaluate OOD detection performance based on PF-ODE log-likelihoods alone, as well as in combination with 18 additional trajectory statistics.

\subsection{Bioinformatic Validation of OOD Levels} 
\label{subsec:bioinformatics_validation}
To confirm the OOD levels of our test datasets, we analyzed the extent to which the validation, CASF2016, and the isolated OOD datasets are represented within the training data. We approximated this degree of representation by quantifying the number and the magnitude of similarities shared between the training complexes and the respective test dataset (Figure \ref{fig:tm_tanimoto_rmsd}). We computed pairwise similarities across all $19'443$ complexes in PDBbind using a combined assessment of protein similarity, ligand similarity and binding conformation similarity. TM-align \cite{Zhang2005} compares protein structures by finding the optimal alignment of their three-dimensional shapes, producing a TM-score ranging from 0 (very dissimilar fold) to 1 (identical fold). Tanimoto similarity \cite{Bajusz2015} is commonly used to measure the similarity between small molecules. Based on comparing chemical fingerprints, this score ranges between 0 (no similarity) and 1 (identical) and identifies compounds with similar structural and chemical properties. The pocket-aligned ligand root mean squared deviation (RMSD) completes the similarity assessment by comparing ligand positioning within aligned protein pockets.  

The distributions of TM-scores (Figure \ref{fig:tm_tanimoto_rmsd}b) reveal that proteins in the OOD datasets are severely underrepresented in the training data, confirming that the pocket-level splitting mechanism effectively isolates proteins with distinct folds. As anticipated, the validation set proteins are well represented withing the training data, while CASF-2016 proteins exhibit moderate similarity. 
Regarding ligand similarities, Tanimoto distributions (Figure \ref{fig:tm_tanimoto_rmsd}a) indicate that the OOD dataset ligands share lower similarity with the training data ligands than the validation set. This confirms that isolating specific protein types during our data splitting automatically induces a moderate shift in the ligand distribution.

\medskip While the Tanimoto and TM scores alone are informative, they are insufficient proxies for protein-ligand interaction similarity. This is because similar protein folds may be associated with entirely different ligands, and structurally similar ligands might bind to completely disparate proteins, thereby resulting in the absence of similarity at the protein–ligand interaction level. Only the combination of these metrics at the complex level, combined with a supplementary assessment of binding conformation similarity, can accurately quantify the structural similarity of protein-ligand complexes.
This information is captured by the aggregated similarity scores (Figure \ref{fig:tm_tanimoto_rmsd}c), calculated as the sum of Tanimoto, TM-align, and inverse ligand RMSD for each pair. This combined metric demonstrates that all OOD datasets share fewer and weaker complex-level similarities with the training data than the validation set. For many training complexes, no representation of the same binding pattern is found in these datasets whatsoever (aggregated similarity scores at zero). This confirms that these datasets are OOD with respect to the training data. The CASF2016 scores confirm that this dataset is moderately OOD and does not reach the same strict level of independence achieved by the OOD datasets.

\begin{figure}[t]
    \centering
    \includegraphics[width=\textwidth]{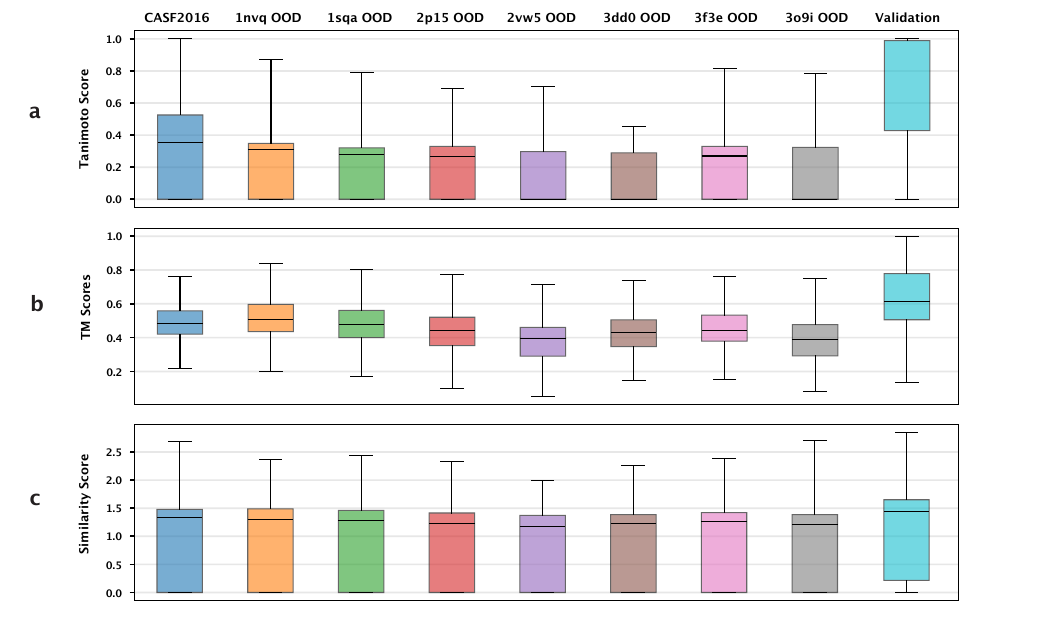}
    \caption{\textbf{Bioinformatic similarity analysis confirms OOD levels:}
    Comparison of the distributions of train-test similarity scores. 
    \textbf{a)} Ligand similarity: Calculated using Tanimoto scores between count-based molecular fingerprints. High scores indicate high similarity (1.0 = identical). 
    \textbf{b)} Protein similarity: Determined by TM-align based on optimal 3D protein structure alignment. High scores indicate high structural similarity (1.0 = identical) 
    \textbf{c)} Aggregated similarity: A composite score $S=max(Tanimoto+TMScore+(1-RMSD), 0)$ was calculated as the sum of Tanimoto similarity, TM-scores, and inverted pocket-aligned ligand root mean squared deviation (RMSD). High scores signify highly similar protein-ligand complexes (3.0 = identical).
    Each box represents the distribution of $N=10,510$ similarity scores. Specifically, these are the scores between each training dataset complex and its most similar counterpart in the respective test dataset, as measured by that specific metric. Boxplots show the median (centre line), 25th–75th percentiles (box), whiskers extend to data points within 1.5 × IQR, outliers are not shown.
    }
    \label{fig:tm_tanimoto_rmsd}
\end{figure}

\subsection{Dataset Log-Likelihood Distributions}
The log-likelihoods assigned to the training, validation, and OOD datasets generally aligned with expectations (Figure \ref{fig:lkhd_histograms}). As anticipated, the training dataset received high log-likelihoods, showing that this dataset fits the distribution learned by the diffusion model. The validation dataset received log-likelihoods slightly lower than the training set, which is expected for unseen samples drawn from the training distribution. The OOD datasets composed of extracted protein family clusters mostly yielded low log-likelihoods, indicating they were correctly identified as OOD. Finally, the CASF2016 dataset produced log-likelihoods intermediate to the training and OOD datasets. This aligns with the nature of the CleanSplit filtering used to ensure the independence of CASF2016: While it removes high-similarity complexes, it does not enforce the strict independence achieved by holding out entire protein families. Among the OOD datasets, the HIV protease family showed the most pronounced distribution shift in log-likelihoods relative to the training data. This finding aligns with biological expectations, as HIV protease complexes possess unique characteristics. These viral proteases contain an active site located between two short identical subunits, flanked by flexible molecular flaps that close upon substrate binding. Inhibitors targeting these enzymes are typically large molecules designed to mimic the proteolytic transition state without being cleaved, effectively blocking the active site. 

\begin{figure}[p]
    \centering
    \includegraphics[width=\textwidth]{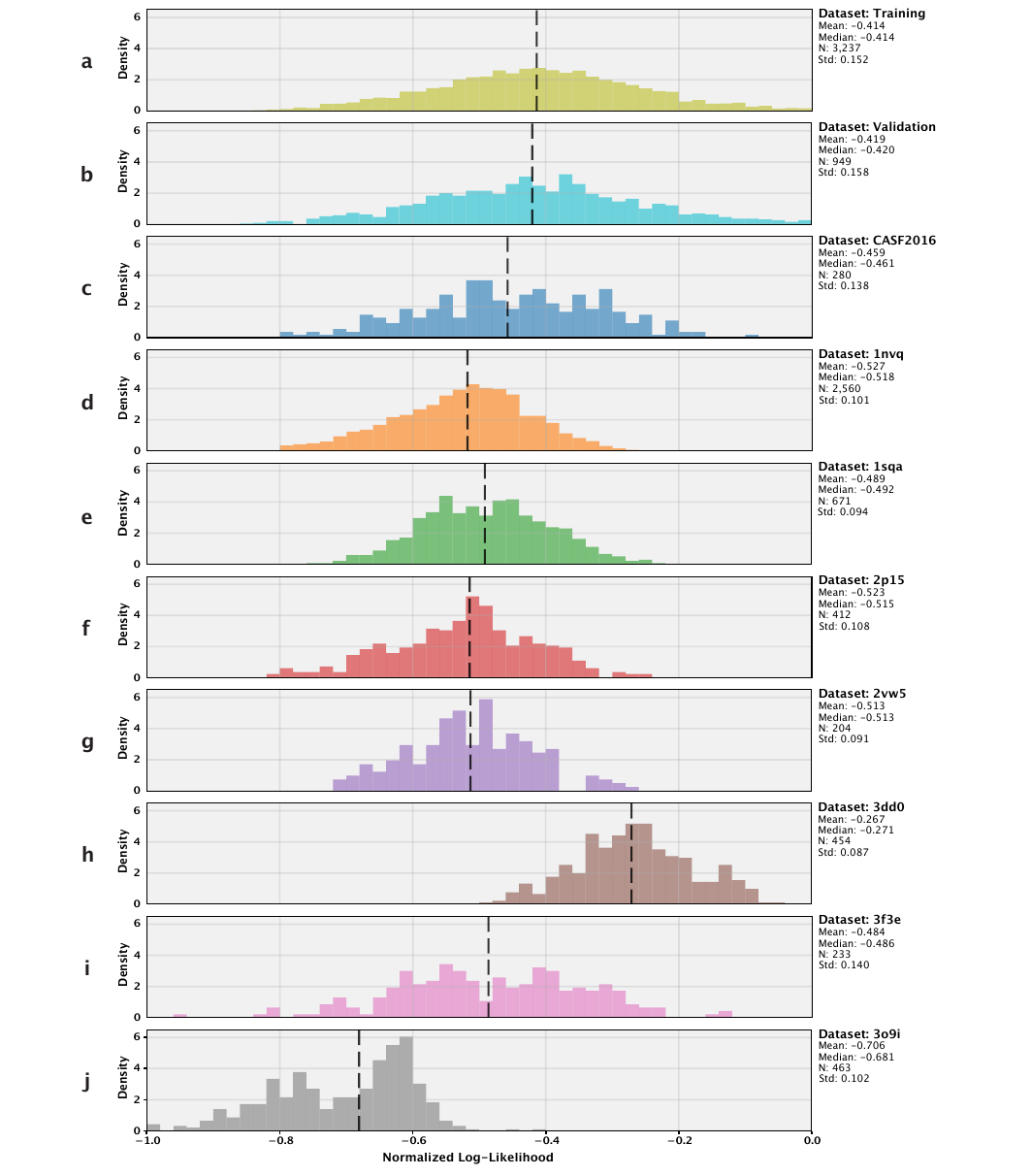}
    \caption{
    \textbf{Out-of-distribution datasets yield lower log-likelihoods:} Distributions of log-likelihoods assigned to protein-ligand complexes belonging to the \textbf{a)} training, \textbf{b)} validation, \textbf{c)} CASF2016 and the out-of-distribution (OOD) datasets \textbf{d)} 1nvq, \textbf{e)} 1sqa, \textbf{f)} 2p15, \textbf{g)} 2vw5, \textbf{h)} 3dd0, \textbf{i)} 3f3e and \textbf{j)} 3o9i. Lower values signify increased deviation from the learned distribution (OOD) and higher values indicate high in-distribution (ID) probability. Means, medians, standard deviations (Std) and number of samples (N) are depicted on the right side of the histograms. Vertical dashed lines indicate the median of the respective log-likelihood distributions. Distributions were subjected to individual outlier removal using the IQR method (1.5×IQR rule) followed by min-max normalization to a [-1, 0] range using global minimum and maximum values across all distributions. The training distribution was randomly subsampled, showing only a representative fraction of all complexes.
    }
    \label{fig:lkhd_histograms}
\end{figure}

\smallskip The spread of the log-likelihoods distributions was notably large across all datasets. This is expected for the inherently diverse training and validation datasets, but also for CASF2016, which is diverse by design. While the OOD datasets of extracted protein families generally showed lower log-likelihood variance, a considerable spread remained, suggesting that these datasets still contain typical as well as atypical complexes. This is likely attributable to our splitting strategy: While we sequestered specific protein families, the ligand space remained uncontrolled. Consequently, OOD pockets may bind a diverse mix of ligands, ranging from familiar to exotic structures, which naturally drives variation in the likelihood scores. 
The impact of ligand diversity on log-likelihood spread becomes evident by analyzing the ligand size distribution and chemical diversity within the datasets (Figures \ref{fig:pocket_ligand_sizes} and \ref{fig:intra_ligand_diversity}). For every dataset, we identified a medoid ligand, defined as the molecule with the smallest sum of distances ($1 - \text{Tanimoto similarity}$) to all other ligands, and computed the distance of all other ligands to this central reference. We found a strong correlation between cluster tightness (median distance to the medoid) and the standard deviation of the log-likelihood distributions (Pearson’s $r = 0.832$). This shows that datasets with diverse sets of ligands exhibited the broadest log-likelihood distributions. The most extreme case is the cluster of transporter protein binding pockets 3f3e, which exhibits a ligand size distribution much broader than all other datasets (Figure \ref{fig:pocket_ligand_sizes}).

\medskip A significant deviation from biological expectation was observed in the log-likelihood distribution of the $\alpha$-carbonic anhydrase dataset (3dd0, Figure \ref{fig:lkhd_histograms}h). This dataset yielded unexpectedly high log-likelihoods, suggesting that it is more in-distribution than the training data itself. We attribute this phenomenon to the complexity bias inherent in deep generative models, where high probability densities are often assigned to simple inputs regardless of their semantic novelty \cite{Nalisnick2018}. The 3dd0 complexes exhibit markedly lower structural complexity (e.g., fewer ligand atoms) compared to the training dataset and all other OOD datasets (Figure \ref{fig:pocket_ligand_sizes}). Consequently, we focus the following evaluation of PF-ODE log-likelihoods on the remaining eight datasets and reserve the discussion of 3dd0 and its successful detection via trajectory features for Section \ref{sec:OOD_classifier_results}.

\subsection{Log-Likelihoods Align with Bioinformatic OOD Levels} We observed a strong alignment between the diffusion model's log-likelihood distributions and the bioinformatic OOD levels (Figure \ref{fig:lkhd_tm_tanimoto_rmsd}). These bioinformatic scores are distributions of similarity scores between dataset complexes, calculated as the sum of Tanimoto, TM-align, and inverted RMSD scores to account for protein, ligand, and binding pose simultaneously (Section \ref{subsec:bioinformatics_validation}). These results demonstrate that datasets sharing fewer and weaker complex-level similarities with the training data generally yield lower log-likelihoods. This direct correspondence confirms that the model's likelihood estimates are chemically grounded, capturing the structural novelty of the underlying molecular interactions.

\begin{figure}[h!]
    \centering
    \includegraphics[width=\textwidth]{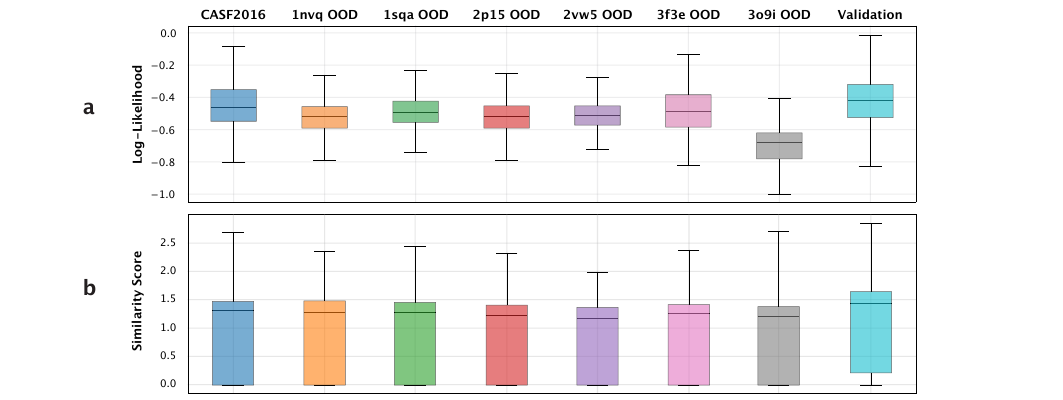}
    \caption{\textbf{Log-likelihood distributions align with bioinformatic similarity metrics:}
    \textbf{a)} Distributions of log-likelihoods assigned by the diffusion model to the validation set (N=949) and to the test datasets  1nvq (N=2560), 1sqa (N=671), 2p15 (N=412), 2vw5 (N=204), 3dd0 (N=454), 3f3e (N=233), 3o9i (N=463) and CASF2016 (N=280). \textbf{b)} The distributions of similarity scores $S=max(Tanimoto+TMScore+(1-RMSD), 0)$ between the training dataset and the test datasets. Scores are calculated as the complex-wise sum of Tanimoto similarity, TM-scores, and inverted pocket-aligned ligand root mean squared deviation (RMSD). High scores signify highly similar protein-ligand complexes (3.0 = identical). Each box represents the distribution of $N=10,510$ similarity scores. Specifically, these are the scores between each training dataset complex and its most similar counterpart in the respective test dataset, as measured by that specific metric. Boxplots show the median (centre line), 25th–75th percentiles (box), whiskers extend to data points within 1.5 × IQR, outliers are not shown.
    }
    \label{fig:lkhd_tm_tanimoto_rmsd}
\end{figure}

\subsection{Log-Likelihoods are Predictive of GEMS Errors} \label{subsec:gems_correlations}
A critical application of OOD detection is to preemptively flag samples where predictive models are likely to fail. If the diffusion model's PF-ODE log-likelihoods accurately reflect the degree of data unfamiliarity, they should predict the performance degradation of independent models trained on the same dataset. This would allow the diffusion model to act as a quality filter, warning users when a new sample is too novel for reliable prediction. To test this capability, we trained the GEMS binding affinity model on the same dataset. This graph neural network processes graph representations of protein-ligand interactions and predicts a scalar binding affinity value. Using the test datasets, we analyzed whether low diffusion likelihoods consistently correspond to high prediction errors. 

\smallskip Across the eight evaluated datasets, we observed a strong correlation between the diffusion model's median log-likelihoods and the predictive performance of GEMS on the same datasets. This relationship was significant both in terms of the coefficient of determination ($R^2$, Pearson Correlation Coefficient $r=0.750$) and the mean absolute error ($r=-0.880$). 
High-likelihood (ID) datasets consistently yielded more accurate predictions (Figure \ref{fig:lkhd_R2_boxplot}), while low likelihood distributions (OOD) showed a steep increase in error (Figure \ref{fig:lkhd_error_heatmaps}). At the individual complex level, this trend follows an approximate exponential behavior: Samples with low likelihoods are prone to large prediction errors, whereas complexes with high likelihoods are predicted accurately with high probability (Figure \ref{fig:lkhd_error_exponential}).

\begin{figure}[]
    \centering
    \includegraphics[width=\textwidth]{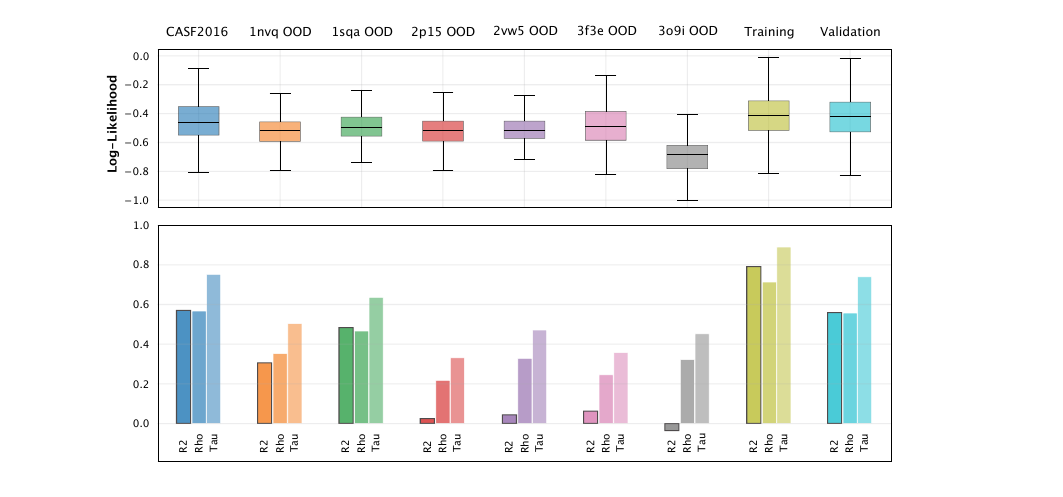}
    \caption{
    \textbf{Correlation between log-likelihoods and GEMS performance:} Comparison of log-likelihood distributions obtained for training, validation, CASF2016 and out-of-distribution (OOD) datasets with performance metrics achieved by the GEMS binding affinity prediction model trained on the same data. The box plot shows the distribution of log-likelihoods assigned to each dataset's protein-ligand complexes by the diffusion model, where lower values signify increased deviation from the learned distribution (OOD) and higher values indicate high in-distribution (ID) probability. The bar plot shows the corresponding performance metrics that GEMS achieved on the same datasets, including R-squared (R2), Kendall (Tau) and Spearman (Rho) rank correlation coefficients. Boxplots show the median (centre line), 25th–75th percentiles (box), whiskers extend to data points within 1.5 × IQR, outliers are not shown.
    }
    \label{fig:lkhd_R2_boxplot}
\end{figure}

\begin{figure}[t]
    \centering
    \includegraphics[width=\textwidth]{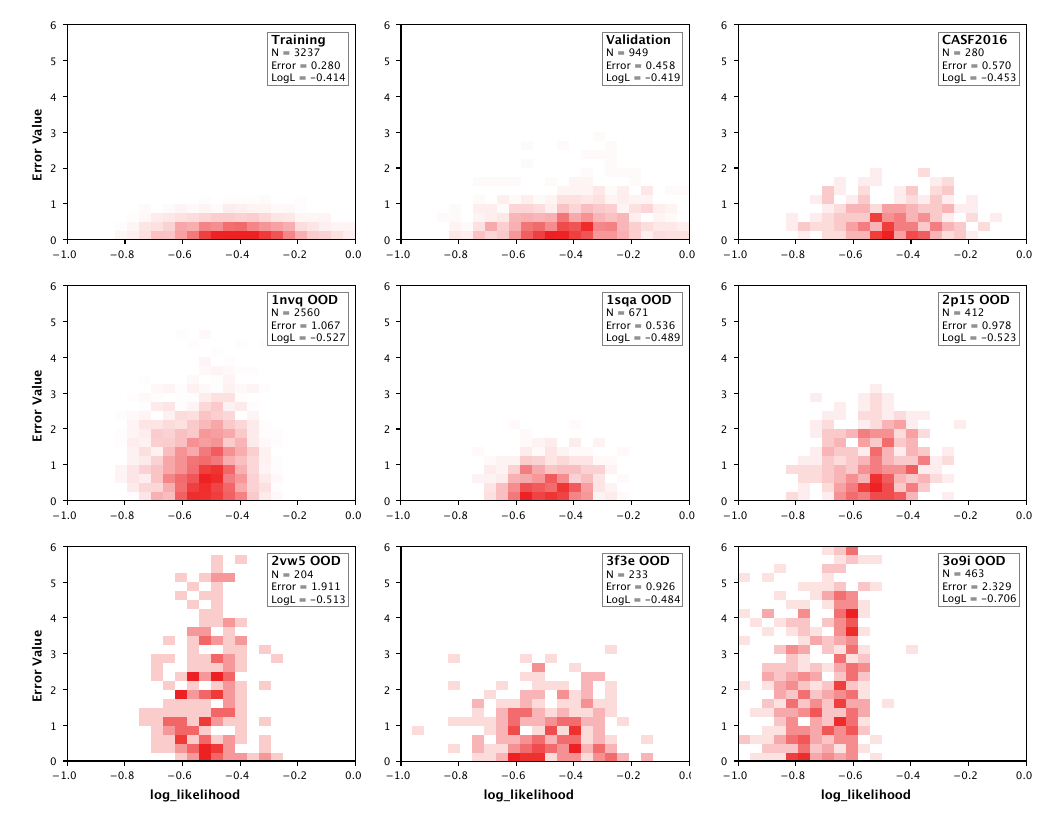}
    \caption{\textbf{Out-of-distribution complexes yield higher GEMS errors:} Heatmaps illustrating the relationship between the log-likelihoods assigned by the diffusion model and the corresponding GEMS binding affinity prediction errors (y-axis) across the training, validation, CASF2016 and out-of-distribution (OOD) datasets. The x-axis represents the log-likelihood for each complex, where lower values signify increased deviation from the learned distribution (OOD) and higher values indicate high in-distribution (ID) probability. The y-axis represents the variance-normalized errors of the GEMS binding affinity prediction. Color intensity represents density, with white areas containing no data. To allow fair comparison across datasets with different label spreads, absolute errors are normalized by the label variance, preventing artificially low error values on narrowly distributed datasets. The heatmaps visually confirm that low log-likelihoods are associated with higher GEMS prediction errors.
    }
    \label{fig:lkhd_error_heatmaps}
\end{figure}

\medskip To strengthen our observation that PF-ODE log-likelihoods are predictive of the performance of other predictive models that have been trained on the same data, we established a formal link between the PF-ODE log-likelihood and the reliability of predictions made by GEMS and show that the log-likelihoods control the error of the downstream model with high probability. 

We interpret the PF-ODE negative log-likelihood $L(x)$ as a task-agnostic typicality score under the diffusion density $q_\phi$ learned from ID complexes.
\[
L(x)\;:=\;-\log q_\phi(x),
\]
 Empirically, this typicality is also predictive of downstream reliability: Complexes with unusually large $L(x)$ (low likelihood under $q_\phi$) tend to incur larger binding-affinity errors under GEMS (Figures~\ref{fig:lkhd_error_heatmaps} and \ref{fig:lkhd_error_exponential}). The key intuition is that if prediction error grows (roughly) monotonically as inputs become less typical, then concentration of $L(x)$ under the ID law immediately yields a high-probability error guarantee. This is the content of the following informal proposition:

\begin{proposition}[Likelihood controls error with high probability]
\label{prop:likelihood-error-informal}
Let $x\sim p_{\mathrm{ID}}$ and define $L(x)=-\log q_\phi(x)$ with typical value $L_{\mathrm{typ}}=\E_{\mathrm{ID}}[L(x)]$.
Assume $L(x)$ has finite ID variance $\Var_{\mathrm{ID}}(L(x))\le \sigma^2$, and that there exists a non-decreasing calibration curve $\phi$ such that $e_\theta(x)\le \phi(L(x))$ for (almost) all $x$.
Then, for any margin $\alpha>0$,
\begin{equation}
\label{eq:likelihood-error-informal}
\mathbb{P}_{\mathrm{ID}}\!\bigl(
e_\theta(x)\;>\;\phi(L_{\mathrm{typ}}+\alpha)
\bigr)
\;\le\;
\frac{\sigma^2}{\alpha^2},
\end{equation}
equivalently $\;\mathbb{P}_{\mathrm{ID}}\!\left(e_\theta(x)\le \phi(L_{\mathrm{typ}}+\alpha)\right)\ge 1-\sigma^2/\alpha^2$.
\end{proposition}

\noindent
In practice, $\phi$ is obtained by fitting a monotone upper envelope to the empirical $(L(x),e_\theta(x))$ scatter on a held-out calibration set (Fig.~\ref{fig:lkhd_error_exponential}), turning PF-ODE likelihoods into a quantitative, label-free risk stratifier for supervised affinity prediction. The supplementary material (Section~\ref{subsec:pf-ode-error-theory}) gives the formal statement (Theorem~\ref{thm:likelihood-error}) and the broader PF-ODE context underlying this bound.

\subsection{Robust OOD Classification via Trajectory Features} \label{sec:OOD_classifier_results}
The strong correlation between diffusion log-likelihoods and bioinformatic similarity across eight test datasets suggests that the PF-ODE log-likelihoods are rooted in the underlying chemical similarity. Nevertheless, they prove insufficient for robust OOD detection in isolation. The log-likelihood distributions exhibit significant overlap with the training data, making accurate classification difficult. Moreover, the metric fails completely on low-complexity interactions due to complexity bias (3dd0, Figure \ref{fig:lkhd_histograms}h). To resolve these limitations and improve separability, we introduce a novel methodology for OOD detection that leverages the full information in a sample's PF-ODE trajectories for accurate classification. Instead of relying solely on the final scalar log-likelihood, we collect 18 additional PF-ODE trajectory statistics that characterize the model's dynamical behavior when transforming an input structure into Gaussian noise. 

\begin{itemize}
    \item Geometric Inefficiency: Path tortuosity and its inverse, path efficiency, characterize the extent to which the trajectory wanders or deviates from a direct path. Increased tortuosity is an indicator of distribution shift and OOD, suggesting the model struggles to maintain a geometrically efficient path in high-dimensional space.
    \item Local Instability and Stiffness: The maximum Lipschitz estimate reads off flow stiffness. High values indicate that nearby states elicit disproportionately different drifts, signaling that the model is operating off-manifold.
    \item Vector Field Activity and Smoothness: The vector-field mean and a spikiness ratio quantify the magnitude and burstiness of corrective pushes required by the model. Mean acceleration records the lack of temporal smoothness, indicating abrupt changes in the flow's velocity.
    \item Energetic Cost: Total flow energy, aggregated along the trajectory path, captures the total directed "work" expended by the model.
\end{itemize}

These additional features show coordinated variation under distribution shift: Stiff, bursty vector fields, higher acceleration and longer, less efficient paths tend to co-occur when processing OOD samples. Therefore, including quantities describing this coordinated movement facilitates robust OOD classification. 
To build a classifier that synthesizes these 19-dimensional feature space into an OOD signal, we fit separate Gaussian Kernel Density Estimators (KDEs) to ID and OOD samples. For new data, we derive a negated Log-Density Ratio (LDR), where larger scores indicate higher OOD confidence. A binary classification is achieved by thresholding, optimized on a validation set (Section \ref{subsec:trajectory}). We refer to this classifier as trajectory-aware LDR classifier.

\paragraph{Trajectory statistics overcome complexity bias}
One of the nine non-training datasets (3dd0, containing $\alpha$-carbonic anhydrases) yielded unexpectedly high PF-ODE log-likelihoods (Figure \ref{fig:lkhd_histograms}). Relying on these scores alone would result in the erroneous classification of this dataset as in-distribution. However, our results demonstrate that augmenting the log-likelihoods with additional trajectory statistics sufficiently enriches the feature space to correct this prediction and flag the dataset as OOD with high confidence. This ensemble of trajectory features enabled the construction of a robust decision boundary between the training data and the $\alpha$-carbonic anhydrase structures, effectively compensating for the lack of discriminative signal in the raw likelihood scores.

\paragraph{Classifier Performance}
To ensure robust performance estimates, we employed a balanced bootstrap evaluation strategy for each test dataset. Because the number of ID decoys is much larger than the number of OOD binders, we keep the OOD test set fixed and repeatedly draw 100 balanced ID test subsets of the same size. For each draw we evaluate the classifier and record AUROC as a measure of ID/OOD separability. 
We observed robust performance on test datasets with strong distribution shifts relative to the training data (Table  \ref{tab:results_traj-auroc-acc} and \ref{tab:traj-ldr-ood}), achieving consistent separation on the most structurally unique OOD data such as HIV-proteases (3o9i), heat-shock-proteins (2vw5) and $\alpha$-carbonic anhydrases (3dd0). However, as anticipated, accuracy diminished on datasets with intermediate OOD characteristics, confirming that these datasets contain complexes that retained moderate similarity to the training data and are therefore hard to separate. A prime example is the CASF2016 benchmark: Originally sampled from the training distribution, this dataset was rendered more independent via CleanSplit filtering. Consequently, it occupies a OOD space between the training data and the OOD datasets, as confirmed by both log-likelihood profiles and bioinformatic similarity analysis (Figure \ref{fig:lkhd_tm_tanimoto_rmsd}). This moderate distribution shift resulted in the low OOD-classification accuracy. Similarity, the transporter protein dataset (3f3e) was predicted with low accuracy. This dataset equally showed intermediate OOD characteristics. Its log-likelihood distribution was very broad with significant overlap with the training log-likelihoods, a finding that was further confirmed with bioinformatic metrics showing a significant degree of retained structural similarity with the training data (Figure \ref{fig:lkhd_tm_tanimoto_rmsd}).

\begin{table}[t]
\centering
\footnotesize
\begin{tabular}{lccc}
\toprule
Dataset & AUROC & Accuracy \\
\midrule
\texttt{casf2016} & $0.708\,{\pm}\,0.042$ & $0.604\,{\pm}\,0.029$ \\
\texttt{1nvq}     & $0.842\,{\pm}\,0.013$ & $0.744\,{\pm}\,0.011$ \\
\texttt{1sqa}     & $0.888\,{\pm}\,0.013$ & $0.787\,{\pm}\,0.014$ \\
\texttt{2p15}     & $0.754\,{\pm}\,0.034$ & $0.693\,{\pm}\,0.027$ \\
\texttt{2vw5}     & $0.939\,{\pm}\,0.033$ & $0.888\,{\pm}\,0.028$ \\
\texttt{3dd0}     & $0.951\,{\pm}\,0.009$ & $0.956\,{\pm}\,0.009$ \\
\texttt{3f3e}     & $0.639\,{\pm}\,0.061$ & $0.598\,{\pm}\,0.035$ \\
\texttt{3o9i}     & $0.963\,{\pm}\,0.006$ & $0.949\,{\pm}\,0.011$ \\
\bottomrule
\end{tabular}
\caption{
\textbf{Trajectory-aware OOD classifier performance:} Comparison of AUROC and accuracy across OOD datasets for the trajectory-aware LDR diffusion-based detector. To obtain a robust and balanced performance estimate despite class imbalance, we employ a bootstrap evaluation strategy. We keep the OOD test set fixed and repeatedly draw 100 balanced ID subsets of the same size from the training data. For each draw we evaluate the classifier, treating OOD as the positive class and report the mean $\pm$ standard deviation over the balanced bootstrap subsets. 
}
\label{tab:results_traj-auroc-acc}
\end{table}

\paragraph{Feature Importance}
To understand which aspects of the PF-ODE trajectory drive separation between ID and OOD complexes, we quantify the importance of each of our $d=19$ trajectory statistics and rank them accordingly (Figure \ref{fig:feature_importances}). Rather than reporting a single heuristic, we use a combined importance score that rewards features which separate ID and OOD strongly and do so consistently across complexes and bootstrap splits.

Several patterns stand out: First, PF-ODE log-likelihood is the single most discriminative coordinate, but it is not the whole story: Multiple trajectory descriptors achieve comparable importance, confirming that OOD shift manifests as a dynamic signature along the probability flow rather than only as a terminal density value. Second, the highest-ranked non-likelihood features cluster into interpretable families: (i) path geometry (tortuosity/efficiency), indicating that OOD inputs induce longer and less direct transport to noise; (ii) field magnitude and burstiness (vf\_l2 mean/max/std and spikiness), indicating larger and more intermittent corrective updates; and (iii) stability and smoothness (Lipschitz and acceleration), indicating stiffer and less temporally regular dynamics. Finally, coupling-related terms (e.g.\ coupling\_consistency and dynamic\_coord\_feature\_coupling) rank non-trivially, suggesting that OOD behavior is not purely geometric: Atypical complexes also perturb how coordinate and identity updates co-evolve across noise scales.

\begin{figure}[t]
    \centering
    \includegraphics[width=\textwidth]{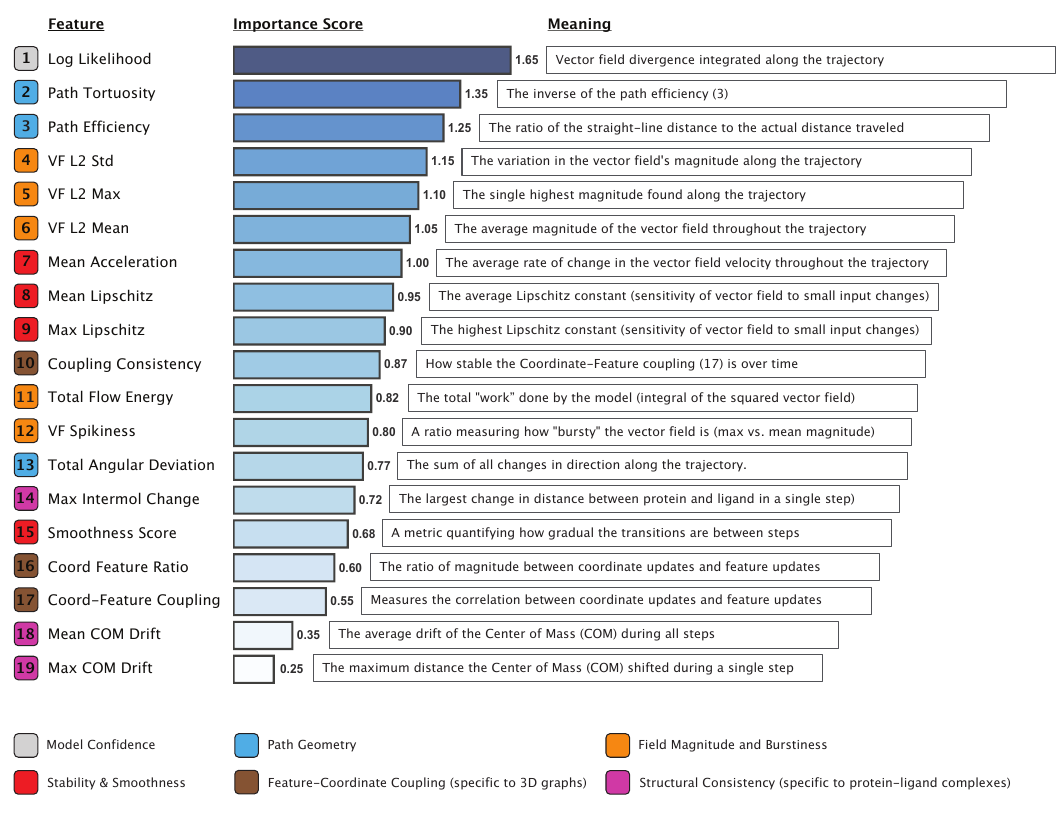}
    \caption{\textbf{Ranking of PF-ODE trajectory features by discriminative power for OOD detection:} Relative importance of 19 trajectory statistics, quantified by a composite score of separation strength (effect size between ID/OOD distributions) and stability under resampling. While log-likelihood is the dominant feature, dynamic characteristics such as path tortuosity, vector field energy, and Lipschitz stability contribute significantly to detection performance. Color indicators classify features into interpretable families (e.g., Path Geometry, Field Magnitude, Stability, Feature-Coordinate Coupling and Structural Consistency). This ranking confirms that OOD samples exhibit distinct dynamic noising trajectories, traversing longer, more energetic, and less stable paths.}
    \label{fig:feature_importances}
\end{figure}

\subsection{GEMS-based Post-Hoc Baselines}\label{sec:gems_baselines}

Due to the novelty of the 3D graph OOD detection task, there is currently a lack of competing models or established benchmarks to serve as baselines for our framework. Therefore, we constructed our own suite of post-hoc baselines by adapting two distinct paradigms from other machine learning domains. First, we implemented a family of classical anomaly detectors operating on the fixed latent representations of the pretrained GEMS encoder to test if OOD signals are captured in static feature spaces. Second, we implemented Rate-In, a state-of-the-art dynamic uncertainty baseline, which probes the stability of GEMS under dropout.

\paragraph{Embedding Space Baseline} 
We treated the GEMS encoder as a static feature extractor and assumed that ID complexes occupy a relatively compact region of the GEMS representation space, and that OOD complexes are either geometrically distant from this region or sparsely supported. For every molecular complex, we computed a summary embedding by averaging the output vectors from multiple stochastic forward passes of GEMS. We then trained five unsupervised anomaly detectors on these standardized ID embeddings, covering distance–based, density–ratio, isolation–based and margin–based approaches:

\begin{itemize}[itemsep=0.5pt]
    \item Mahalanobis Distance: Measures the distance of a sample from the ID class mean, accounting for covariance.
    \item k-Nearest Neighbors (k-NN): Scores anomaly based on the average distance to the nearest ID training examples.
    \item Local Outlier Factor (LOF): Detects samples with lower local density compared to their neighbors.
    \item Isolation Forest: Assigns higher anomaly scores to points that are easily isolated (few splits) in random decision trees.
    \item One-Class SVM: Learns a decision boundary (margin) that encloses the ID data.
\end{itemize}

All classifiers were trained on the same fixed GEMS latent representations. Performance was estimated following the same rigorous evaluation protocol as with our trajectory-aware OOD detection classifier, using balanced bootstrap test subsets to ensure robustness against class imbalance (Section \ref{sec:OOD_classifier_results}). As a main performance readout and robust indicator of detector quality, we compute the threshold-free AUROC (Table \ref{tab:ood-auroc-combined}). Additionally, we calibrate decision thresholds for each OOD dataset on held-out validation sets (maximizing F1-score). We then report the performance of the detector with the highest AUROC for each dataset (Table \ref{tab:gems-ood}).

\paragraph{Adaptive Uncertainty Baseline (Rate-In)} 
We implemented Rate-In \cite{zeevi2025rate}, a dynamic baseline that is built on the concept of Monte Carlo dropout. Instead of a fixed global dropout rate, Rate-In applies adaptive, per-sample, per-layer dropout rates chosen to preserve information flow through the network. It works under the assumption that ID samples allow for relatively high dropout rates without disrupting the computation, whereas OOD inputs are more fragile and loose information under dropout perturbations. 
We integrated this adaptive scheme directly into the GEMS binding affinity predictor. For each molecular complex, the method iteratively optimizes the dropout intensity at every network layer, seeking the maximum dropout level that still preserves stable information flow. This process generates a sample-specific robustness profile. The final anomaly score is constructed by aggregating three distinct signals: The standard predictive variance (epistemic uncertainty), the heterogeneity of the optimized dropout rates across layers, and the computational effort (convergence difficulty) required to find these stable rates. High instability in these metrics indicates that the predictive model struggles to process the input structure, suggesting the input is OOD.

\paragraph{Performance Evaluation} 
These baselines establish a rigorous reference for evaluating our trajectory-aware method. While Rate-In provides a sophisticated measure of predictive uncertainty, our empirical analysis showed that it offered lower discrimination accuracy compared to the trajectory–aware OOD classifier approach (Table \ref{tab:ood-auroc-combined}). This demonstrates that probing the stability of a discriminative model is inferior to the sensitivity of the diffusion-based approach using multiple PF-ODE trajectory features. Similarly, the embedding space baseline showed lower OOD detection accuracy across most datasets, showing that the latent space features from the GEMS encoder possess less OOD separability than the descriptors extracted from the PF-ODE trajectory. Interestingly, both the embedding-space baseline and Rate-In outperformed our method on the transporter protein dataset (3f3e). This suggests that for certain classes of distribution shift, the static latent representation of GEMS may capture specific OOD characteristics that are more discriminative than the PF-ODE flow dynamics of our diffusion model.

\begin{table}[t]
\centering
\small
\begin{tabular}{lccc}
\toprule
Dataset & Traj-LDR (diffusion) & GEMS (embedding) & Rate-In \\
\midrule
\texttt{casf2016} & $0.708\,{\pm}\,0.042$ & $0.554\,{\pm}\,0.019$ & $0.444\,{\pm}\,0.041$ \\
\texttt{1nvq}     & $0.842\,{\pm}\,0.013$ & $0.804\,{\pm}\,0.007$ & $0.581\,{\pm}\,0.013$ \\
\texttt{1sqa}     & $0.888\,{\pm}\,0.013$ & $0.928\,{\pm}\,0.008$ & $0.664\,{\pm}\,0.024$ \\
\texttt{2p15}     & $0.754\,{\pm}\,0.034$ & $0.721\,{\pm}\,0.017$ & $0.541\,{\pm}\,0.039$ \\
\texttt{2vw5}     & $0.939\,{\pm}\,0.033$ & $0.659\,{\pm}\,0.033$ & $0.439\,{\pm}\,0.047$ \\
\texttt{3dd0}     & $0.951\,{\pm}\,0.009$ & $0.728\,{\pm}\,0.019$ & $0.337\,{\pm}\,0.029$ \\
\texttt{3f3e}     & $0.639\,{\pm}\,0.061$ & $0.919\,{\pm}\,0.009$ & $0.741\,{\pm}\,0.032$ \\
\texttt{3o9i}     & $0.963\,{\pm}\,0.006$ & $0.922\,{\pm}\,0.010$ & $0.859\,{\pm}\,0.024$ \\
\bottomrule
\end{tabular}
\caption{
\textbf{Comparison of AUROC (mean $\pm$ standard deviation) across OOD datasets} for the trajectory-aware LDR diffusion-based detector (Traj-LDR), the embedding-space GEMS baseline, and the Rate-In baseline. All values are taken from Tables~\ref{tab:traj-ldr-ood}, \ref{tab:gems-ood} and \ref{tab:ratein-ood}. The \texttt{validation} split used only for Rate-In is omitted here.
}
\label{tab:ood-auroc-combined}
\end{table}

\paragraph{Key distinctions from baselines} 
The bioinformatic similarity analysis (Sections~\ref{subsec:bioinformatics_validation}) and the GEMS-based post-hoc OOD detection baselines (Section \ref{sec:gems_baselines}) provide a strong reference for our trajectory-aware OOD detection. However, our diffusion-based OOD detector occupies a fundamentally different region in the design space than the deterministic baselines. The diffusion model is trained in a purely generative fashion. It only observes bare three-dimensional coordinates and discrete atom/residue identities, and its objective is to approximate the joint data distribution. It never sees binding affinity labels, similarity scores, or any task-specific supervision. By contrast, out baselines are heavily endowed with biochemical priors: The bioinformatic OOD approximations are hand-crafted around well-established notions of structural and chemical similarity like TM-scores, Tanimoto scores and ligand RMSD values. In the Embedding Space baseline, OOD decisions are based on a static snapshot of the GEMS representation geometry. Similarly, the Rate-In baseline makes use of the pretrained weights of the GEMS model and derives OOD scores by probing prediction stability under dropout perturbation. 

\medskip In this sense, the diffusion model is substantially more structure-agnostic than the baselines, yet it is asked to solve a more demanding problem (full generative modeling) from strictly weaker supervision. Nevertheless, the OOD detection with diffusion trajectory features outperforms all baselines in our task of recognizing OOD datasets of isolated protein families (Table \ref{tab:ood-auroc-combined}).

Moreover, the scope of the baseline methods employed here is inherently limited. The bioinfomatic approximation of OOD scores requires exhaustive pairwise comparisons to construct similarity matrices. Both the Rate-In and the Embedding Space baseline rely on a high-quality GEMS encoder trained directly on the binding affinity task. Consequently, their applicability is restricted to supervised settings where such pretrained models exist. In contrast, our OOD detector based on diffusion PF-ODE trajectory is trained in a generative fashion and requires no labels. It only exploits the diffusion model's own internal dynamics and thus provides a versatile, label-free mechanism for OOD detection that is applicable to diverse data types and training regimes.

\clearpage
\section{Discussion}

We have introduced a generative framework for detecting out-of-distribution (OOD) structures in irregular 3D graph data, demonstrating that a diffusion model trained purely on raw geometry and discrete biochemical identities recovers a principled, task-agnostic OOD signal within its dynamic noising trajectories. Our methodology employs a unified continuous diffusion process and posterior-mean interpolation to establish a stable SE(3)-equivariant denoiser whose probability-flow ODE (PF-ODE) trajectories provide a per-complex notion of typicality with respect to the training distribution. The central methodological advancement is the augmentation of scalar likelihoods with 18 additional trajectory statistics. By capturing path geometry, flow stiffness, vector-field instability, and energetic cost, we demonstrate that OOD structures induce distinct dynamics in their diffusion trajectories that facilitate robust detection, effectively bypassing the complexity bias that often confounds methods based on likelihoods alone. This approach successfully identifies bioinformatically validated OOD protein-ligand interactions and outperforms classical post-hoc baselines despite working in a completely unsupervised setting.

\smallskip The practical utility of this unsupervised approach is underscored by the massive disparity between available structural data and curated labels. With the amount of structural data growing at an exponential rate relative to labeled datasets like PDBbind, the ability to derive OOD signals without supervision is critical. To our knowledge, this work represents the first report of unsupervised OOD detection specifically for 3D geometric graph data, a modality where detection via simple statistical checks or similarity computations is inherently impractical due to the complex combination of topology, 3D geometry, and discrete features. Instead, effective detection requires a joint assessment of these features; a task our framework accomplishes by explicitly modeling the probability density of this joint distribution. Our results demonstrate that these diffusion-based estimates can provide an a priori assessment of a model's reliability, forecasting prediction errors before a complex is processed by a downstream predictor. This provides a critical safety measure for deployed machine learning models, warning users when input data points are atypical and the model's output may be unreliable.

\medskip Beyond practical safety, this work contributes to the ongoing debate between generalization and memorization in structural bioinformatics. While modern models often report high benchmark performance, such claims are frequently accompanied by concerns that performance may largely stem from memorizing recurring patterns shared between training and test data. We propose that OOD quantification can act as a "certificate" for benchmark datasets and the connected generalization claims. By explicitly relating prediction errors to trajectory-based OOD scores, we can disentangle true generalization from dataset-specific memorization. A model exhibiting a shallow increase in error as it moves into OOD territory demonstrates genuine generalization to data underrepresented in its training set, whereas a steep increase indicates performance rooted in memorized training motifs. This enables a direct comparison across competing architectures, moving beyond single-point benchmark values to evaluate a model's true generalization capacity.

\smallskip While this work provides a proof of concept for PF-ODE trajectory features, significant potential remains to improve the diffusion model's distribution learning to further increase sensitivity to subtle chemical and structural nuances.
Nevertheless, our findings demonstrate that the generative modeling of 3D graphs inherently captures the "essence" of the data distribution. The emergence of robust OOD signals as a byproduct of the generative learning process, without task-specific labels or hand-crafted priors, offers a versatile mechanism for safety and reliability in geometric deep learning. This framework is not restricted to molecular complexes but is inherently generalizable to any modality amenable to continuous diffusion, offering a general recipe for OOD quantification and error forecasting. Our method should be viewed as one concrete instantiation of a broader OOD detection blueprint: The core requirement is simply a generative dynamical model, a diffusion or flow-matching model, trained to represent the in-distribution. Once such a model is available, the induced probability-flow dynamics provide a natural coordinate system in which to build trajectory-level diagnostics. Researchers can design score functions and summary statistics that probe the structural regularities relevant to their specific application, detecting OOD samples by identifying trajectories whose behavior is atypical under the learned training distribution.

\newpage
\section{Method}
\subsection{Datasets} \label{met:Datasets}
Existing diffusion models for protein–ligand interactions such as DiffSBDD \cite{Schneuing2024} and TargetDiff \cite{TargetDiff} were trained on datasets of experimental structures extracted from the Protein Data Bank (PDB), frequently augmented with artificial complexes generated by docking tools \cite{Francoeur2020}. The PDBbind database (v.2020) represents the subset of all available experimental protein-ligand complex structures for which binding affinity labels are available \cite{Wang2004, Liu2015}. This dataset (N=19'443) served as the main data resource to train and evaluate models in this work.

To evaluate the OOD detection capability of our diffusion model, we split the PDBbind dataset into a training set and several test sets with varying biological uniqueness using PLINDER’s pocket-level clustering based on Local Distance Difference Test lDDT (cluster type \textit{pocket lddt 50 community})\cite{PLINDER}. Seven proteins were selected to represent diverse protein families (PDB IDs: 1nvq, 1sqa, 2p15, 2vw5, 3dd0, 3f3e, 3o9i). For each representative protein, all complexes from its cluster were assigned to a separate test set, thereby removing entire protein families from the training dataset. To be able to use the CASF2016 benchmark \cite{Su2019} as an additional test dataset, we applied CleanSplit filtering \cite{GEMS} to the training data to remove similarities with CASF2016. We then further split the training data randomly into final training and validation set using a ratio of 90/10. This process resulted in the following datasets:

\begin{itemize}[itemsep=0.5pt]
    \item Training dataset (N=10'510): Dataset containing 90\% of the remaining complexes after OOD dataset extraction and CleanSplit filtering.
    \item Validation dataset (N=1167): Dataset containing 10\% of the remaining complexes after OOD dataset extraction and CleanSplit filtering.
    \item CASF2016 benchmark (N=285): Diverse dataset with protein-ligand complexes from 57 different protein families. Data leakage reduced trough CleanSplit filtering.
    \item 1nvq OOD (N=2714): Cluster of serine/threonine-protein kinase binding pockets
    \item 1sqa OOD (N=736): Cluster of Urokinase-type plasminogen activator pockets
    \item 2p15 OOD (N=462): Cluster of estrogen receptor binding pockets
    \item 2vw5 OOD (N=207): Cluster of molecular chaperone HSP82 binding pockets
    \item 3dd0 OOD (N=475): Cluster of carbonic anhydrase binding pockets
    \item 3f3e OOD (N=391): Cluster of transporter protein binding pockets
    \item 3o9i OOD (N=469): Cluster of HIV-1 protease binding pockets
\end{itemize}

This dataset split was employed for both the diffusion model and the GEMS binding affinity prediction model. The diffusion model was trained on the training dataset, with validation losses computed throughout the training process using the validation dataset. All non-training datasets (validation, CASF2016, and OOD sets) were subsequently used to evaluate the model's ability to differentiate and rank datasets according to their OOD level relative to the training data. GEMS was trained exclusively on the training dataset using 5-fold cross-validation (CV). All other datasets served as test sets to evaluate generalization capability to datasets of varying distribution shifts, ranging from ID (e.g. the validation set) to highly OOD (e.g. the HIV proteases).

\subsection{Diffusion Model} \label{sec:diffusion_model}

Our diffusion model reuses only the EGNN backbone architecture from DiffSBDD~\cite{Schneuing2024}; all other components, ranging from the diffusion training infrastructure, to metrics and evaluation components and sampling procedures (covering also the PF-ODE implementation) are specific to this work. 

\subsubsection{Background and Motivation}
\paragraph{Continuous diffusion and probability--flow}
Diffusion models specify a family of corrupted observations $\{\vect{x}_t\}_{t\in[0,1]}$ linked by a forward SDE (typical setup of \cite{EDM})
\begin{equation}
\dd \vect{x} \;=\; f(\vect{x},t)\,\dd t \;+\; g(t)\,\dd \vect{w}, \qquad \vect{w}\text{ a standard Wiener process},
\end{equation}
and learn a time-indexed score field $s(\vect{x},t)=\nabla_{\vect{x}}\log p_t(\vect{x})$ to reverse the corruption. Under mild regularity, the \emph{probability--flow ODE} (PF-ODE) $\dot{\vect{x}}=\tilde{f}(\vect{x},t)$ transports samples deterministically while matching the SDE marginals at every $t$; its drift is a functional of the score. For variance-exploding (VE) designs with $f\equiv 0$ and scalar noise scale $\sigma(t)$, the conditional $p(\vect{x}_t\!\mid\!\vect{x}_0)=\Normal(\vect{x}_0,\sigma(t)^2 I)$ and the PF-ODE drift is proportional to the score,
\begin{equation}
\dot{\vect{x}}(t) \;=\; -\tfrac{1}{2}g(t)^2 \, s(\vect{x}_t,t),
\end{equation}
yielding an exact change-of-variables likelihood for any clean input $\vect{x}_0$,
\begin{equation}
\label{eq:pf_ode_likelihood}
\log p(\vect{x}_0) \;=\; \log p_1(\vect{x}_1) \;+\; \int_{0}^{1}\!\tr\!\left(\frac{\partial \tilde{f}}{\partial \vect{x}}(\vect{x}(t),t)\right)\dd t,
\end{equation}
where $\vect{x}(1)$ is the terminal state and $p_1$ is the induced terminal density (Gaussian in the VE case). Equation~\eqref{eq:pf_ode_likelihood} shows that diffusion models are not only samplers but also differentiable density models once a numerically stable divergence estimator is in place.

\paragraph{Denoising as deconvolution (posterior mean)}
In the VE setting, corruption is a Gaussian convolution $\vect{x}_t=\vect{x}_0+\sigma(t)\,\epsilon$ with $\epsilon\sim\Normal(0,I)$. Recovering $\vect{x}_0$ from $\vect{x}_t$ is a classical \emph{deconvolution} problem whose Bayes estimator under squared loss is the posterior mean:
\begin{equation}
\label{eq:posterior_mean}
\widehat{\vect{x}}_0(\vect{x}_t,t) \;=\; \E[\vect{x}_0\mid \vect{x}_t,t] \;=\; \arg\min_{\vect{y}}\; \E\!\left[\|\vect{y}-\vect{x}_0\|_2^2 \,\middle|\, \vect{x}_t,t\right].
\end{equation}
The conditional score is \emph{affine} in $\vect{x}_0$,
\begin{equation}
\label{eq:ve_score}
s(\vect{x}_t,t\mid \vect{x}_0) \;=\; \frac{\vect{x}_0 - \vect{x}_t}{\sigma(t)^2},
\end{equation}
so taking expectation over $p(\vect{x}_0\!\mid\!\vect{x}_t,t)$ gives
\begin{equation}
\label{eq:interp_score}
\hat{s}(\vect{x}_t,t) \;=\; \E[s(\vect{x}_t,t\mid \vect{x}_0)\mid \vect{x}_t,t]
\;=\; \frac{\widehat{\vect{x}}_0(\vect{x}_t,t)-\vect{x}_t}{\sigma(t)^2}.
\end{equation}
Thus, learning any mechanism that reliably produces the posterior mean clean signal \emph{directly delivers} the score needed for both sampling and PF-ODE likelihoods.

\subsubsection{Unified continuous diffusion for molecular coordinates and identities}

\paragraph{Contribution}
We introduce a unified \emph{continuous} diffusion model for molecules that treats both 3D geometry and discrete chemical identities in a single Euclidean state: we learn L2-normalized prototypes for the categorical alphabet and run the same VE diffusion on the concatenated coordinates$\;\|\;$embeddings, converting class logits to scores via \emph{posterior–mean interpolation}. This yields a single, end-to-end objective and avoids the collapse pathologies of direct score regression on the categorical channel.

\paragraph{Background: Score interpolation for categorical data}
Under VE corruption $\vect{x}_t=\vect{x}_0+\sigma(t)\,\epsilon$ with $\epsilon\sim\mathcal{N}(0,\mathbf{I})$, the conditional score is an affine function of the posterior mean of the clean variable:
\begin{equation}
\label{eq:ve_score}
s(\vect{x}_t,t)\;=\;\frac{\E[\vect{x}_0\mid\vect{x}_t,t]-\vect{x}_t}{\sigma(t)^2}.
\end{equation}
For a categorical variable represented by L2-normalized prototypes $\vect{e}_1,\dots,\vect{e}_V\in\R^d$ and a noisy embedding $\vect{x}_t$, if a denoiser outputs logits $\vect{z}$ with class posteriors $p(i\mid\vect{x}_t,t)=\softmax(\vect{z})_i$, then the posterior mean over the clean embedding is the finite mixture
\begin{equation}
\label{eq:embedding_posterior_mean}
\widehat{\vect{e}}_0(\vect{x}_t,t)\;=\;\sum_{i=1}^{V} p(i\mid \vect{x}_t,t)\,\vect{e}_i,
\end{equation}
and the corresponding \emph{interpolated score} for the categorical block follows by plugging \eqref{eq:embedding_posterior_mean} into \eqref{eq:ve_score}:
\begin{equation}
\label{eq:interp_score}
\hat{s}(\vect{x}_t,t)\;=\;\frac{\widehat{\vect{e}}_0(\vect{x}_t,t)-\vect{x}_t}{\sigma(t)^2}.
\end{equation}
This “score via posterior mean” viewpoint enables standard cross-entropy training on logits while keeping the dynamics continuous in time and space (cf.\ CDCD, \cite{Dieleman2022}).

\paragraph{Learning molecular embeddings without collapse}
Direct score matching on the categorical prototypes is ill-posed: with Gaussian corruption around each prototype, a degenerate joint optimum collapses all $\{\vect e_i\}$ to a single point, making the noise perfectly predictable while erasing identity. We therefore \emph{train the categorical path with cross-entropy on logits} and \emph{derive the score} via \eqref{eq:embedding_posterior_mean}–\eqref{eq:interp_score}. Concretely, let $y\in\{1,\dots,V\}$ be the clean token with prototype $\vect e_y$, sample $t\sim p(t)$ and $\epsilon\sim\mathcal N(0,\mathbf I)$, and form the noisy embedding
\[
\vect x_t \;=\; \vect e_y + \sigma(t)\,\epsilon.
\]
The denoiser produces a normalized feature $\widehat{\vect h}_\theta(\vect x_t,t)\!\in\!\R^d$ and cosine logits
\begin{equation}
\label{eq:logits_cosine}
z_i(\vect x_t,t)\;=\;\kappa(t)\,\langle \vect e_i,\;\widehat{\vect h}_\theta(\vect x_t,t)\rangle,
\qquad \|\vect e_i\|_2=\|\widehat{\vect h}_\theta\|_2=1,
\end{equation}
with either a \emph{theory-aligned} temperature $\kappa(t)=\sigma(t)^{-2}$ (matching the Bayes posterior for unit-norm prototypes) or a fixed $\kappa\!=\!1$ for simplicity.\footnote{We enforce the unit-sphere constraints by projection after each update: $\vect e_i\leftarrow \vect e_i/\|\vect e_i\|_2$ and $\widehat{\vect h}_\theta \leftarrow \vect h_\theta/\|\vect h_\theta\|_2$.}
Our primary loss is standard hard-label cross-entropy
\begin{equation}
\label{eq:cat_loss_hard}
\mathcal L_{\text{cat}}(\theta,\{\vect e_i\})\;=\;
\E_{y,t,\epsilon}\Big[-\log \operatorname{softmax}(\vect z(\vect x_t,t))_{y}\Big],
\end{equation}
which learns posteriors $p_\theta(i\mid\vect x_t,t)=\softmax(\vect z)_i$ and, through \eqref{eq:embedding_posterior_mean}–\eqref{eq:interp_score}, yields the interpolation-derived categorical score.

These cosine logits remove scale indeterminacy (no norm inflation), enlarge inter-class margins on the sphere, and stabilize the interpolation-derived score, yielding a compact, regularizer-free recipe for joint training of coordinates and identities.

\paragraph{Unified diffusion over coordinates $\|\,$embeddings}
We concatenate continuous 3D coordinates with the learned categorical embeddings and run the same VE diffusion across the entire state. This brings three concrete advantages:
\emph{(i) Principled estimator.} Denoising is deconvolution: coordinates are regressed to their clean values; categorical embeddings use the exact finite average \eqref{eq:embedding_posterior_mean}. The resulting score \eqref{eq:interp_score} is consistent by construction.
\emph{(ii) Numerical stability.} Cross-entropy over a small chemical alphabet (tens of atom/residue types) is well-conditioned and avoids the small-$\sigma$ gradient blow-ups that can afflict direct score regression; L2 normalization bounds activations and improves Lipschitz behavior across noise scales.
\emph{(iii) Exact finite averaging.} Unlike language modeling, $V$ here is small, so \eqref{eq:embedding_posterior_mean} is computed \emph{exactly} at every node and time step—preserving chemical similarity (nearby classes $\Rightarrow$ nearby embeddings) while representing genuine superpositions at intermediate noise levels.

\paragraph{Geometry and symmetry}
Molecular validity depends on geometry and symmetry. Our diffusion state always enforces translation invariance by projecting the coordinate block to the center-of-mass (COM)-free subspace per complex; the same projection is applied to all coordinate noise, drifts, and divergence probes. An SE(3)-equivariant graph denoiser updates coordinates and predicts logits conditioned on noisy coordinates\,$\|\,$embeddings, allowing geometric cues to inform identity and vice versa. Because \eqref{eq:interp_score} expresses the PF-ODE drift through posterior means, the same denoiser underpins both generation and likelihood evaluation without changing objective or architecture.

\paragraph{Consequence for likelihoods and OOD}
Unifying geometry and identity within a continuous diffusion equipped with exact posterior-mean interpolation yields a \emph{single}, self-consistent PF-ODE whose pathwise divergence and terminal Gaussian define a tractable, per-complex log-likelihood via \eqref{eq:pf_ode_likelihood}. Likelihoods accumulate information across a continuum of noise scales, furnishing a sensitive typicality score for out-of-distribution detection in molecular settings where categorical spaces are small and geometric constraints are paramount.

\subsubsection{Diffusion Model Setup}
\paragraph{State, symmetries, and encoders}
Protein-ligand complexes were mapped to fully-connected graphs including ligand atoms and C$\alpha$ atoms of all amino acids with atoms less then \SI{5}{\angstrom} from any ligand atom. We represent a batched complex by nodes $i=1,\dots,N$ with coordinates $\vect{r}_i\in\R^3$ and categorical labels $y_i\in\{1,\dots,V\}$ representing atom/residue types. Edge features were initialized as the Euclidean distance between the connected nodes. Two learnable encoders map atom/residue types to L2-normalized embeddings $\vect{e}_{y_i}\in\R^d$, and we define the per-node state $\vect{x}_i=[\vect{r}_i\;\|\;\vect{e}_{y_i}]\in\R^{3+d}$. For each complex in the batch we remove global translation by projecting the coordinate block to the center-of-mass (COM)-free subspace at every step; the same projection is applied to all coordinate noise and updates. This enforces translation invariance and reduces the effective coordinate degrees of freedom.

\paragraph{Noise schedule and preconditioning}
We use a monotone VE schedule $\sigma:[0,1]\to (\sigma_{\min},\sigma_{\max}]$; in practice we parameterize $\sigma$ so that $\log \sigma$ is approximately linear in $t$ and is easily invertible for sampler step allocation. Following EDM-style preconditioning, inputs are scaled by $c_{\text{in}}(t) = (\sigma(t)^2+\sigma_{\mathrm{data}}^2)^{-1/2}$ and outputs by $c_{\text{out}}(t) = \sigma(t)\sigma_{\mathrm{data}}(\sigma(t)^2+\sigma_{\mathrm{data}}^2)^{-1/2}$ to equalize dynamic range across noise levels.

\paragraph{Equivariant denoiser}
The denoiser is an EGNN that consumes the concatenated state with a scalar noise/time conditioning and returns two heads per node: a clean coordinate estimate $\widehat{\vect{r}}_{0,i}$ and logits $\vect{z}_i\in\R^{V}$. Using $\vect{z}_i$ we form probabilities $p_i=\softmax(\vect{z}_i)$ and the posterior-mean clean embedding
\begin{equation}
\widehat{\vect{e}}_{0,i} \;=\; \sum_{k=1}^{V} p_i(k)\,\vect{e}_k.
\end{equation}
Collecting per-node predictions, the denoiser yields $\widehat{\vect{x}}_0 = [\widehat{\vect{r}}_0 \;\|\; \widehat{\vect{e}}_0]$.

\paragraph{Training objective}
At a sampled noise level $t$ we construct noisy inputs $\vect{x}_t=\vect{x}_0+\sigma(t)\epsilon$ with $\epsilon\sim\Normal(0,I)$, enforcing COM-free noise in the coordinate block. The loss combines a coordinate regression term with a categorical cross-entropy:
\begin{equation}
\label{eq:loss}
\mathcal{L} \;=\; \lambda_{\mathrm{coord}}\;\E\!\left[ w(t)\,\|\widehat{\vect{r}}_0 - \vect{r}_0\|_2^2 \right]
\;+\; \lambda_{\mathrm{CE}}\;\E\!\left[ \sum_{i=1}^N \mathrm{CE}(\vect{z}_i, y_i) \right]
\;+\; \lambda_{\mathrm{geom}}\;\mathcal{L}_{\mathrm{geom}},
\end{equation}
where $w(t)$ is the EDM weight that flattens the effective noise-level distribution, and $\mathcal{L}_{\mathrm{geom}}$ optionally encodes geometric regularizers (e.g., connectivity, radius of gyration, cycle incentives). L2 normalization of the embedding tables stabilizes learning and empirically prevents collapse without additional penalties.

\paragraph{Reverse-time sampling}
Sampling follows the reverse SDE with Euler–Maruyama integration coupled with COM projection after each step. The drift uses the denoiser’s effective clean prediction and the schedule derivatives (computed by automatic differentiation on $\sigma$). This produces coordinate trajectories and categorical superpositions that resolve as noise decays.

\paragraph{Probability--flow ODE and likelihood} \label{subsec:pf-ode}
For the PF-ODE, we use the VE correspondence that yields a drift proportional to the interpolated score. In particular, with \eqref{eq:ve_score}--\eqref{eq:interp_score} and the standard VE choice $g(t)=\sqrt{2}\,\sigma'(t)$ one obtains
\begin{equation}
\label{eq:pf_drift}
\dot{\vect{x}}(t) \;=\; -\tfrac{1}{2}g(t)^2\,\hat{s}(\vect{x}_t,t)
\;=\; \alpha(t) \left(\vect{x}_t - \widehat{\vect{x}}_0(\vect{x}_t,t)\right),
\qquad \alpha(t) := \frac{g(t)^2}{2\,\sigma(t)^2},
\end{equation}
where $\widehat{\vect{x}}_0$ uses the coordinate head and the class-probability–weighted mean embedding. We integrate \eqref{eq:pf_drift} forward from $t{=}0$ (clean) to $t{=}1$ (noisy) with Heun's method in the COM-free subspace.

The log-density is then computed by \eqref{eq:pf_ode_likelihood}. The terminal density $p_1(\vect{x}_1)$ factors over nodes and splits into a coordinate block and an embedding block. In the coordinate block, the COM constraint reduces the degrees of freedom of a complex with $n$ atoms by $3$ (in 3D), leading to
\begin{equation}
\log p_{1,\mathrm{coord}} \;=\; -\frac{1}{2\sigma(1)^2}\|\vect{r}_1\|_2^2 \;-\; \frac{(3n-3)}{2}\log\!\big(2\pi\sigma(1)^2\big),
\end{equation}
while the embedding block is isotropic Gaussian in $\R^{d}$ per node (with $d$ the embedding dimension). Summing the two yields $\log p_1(\vect{x}_1)$.

\paragraph{Divergence via Hutchinson under COM constraints}
We estimate the divergence in \eqref{eq:pf_ode_likelihood} with a Hutchinson trace estimator. Let $\epsilon\sim\Normal(0,I)$ be a probe in the joint state; project its coordinate block to the COM-free subspace to match the dynamics. A single vector–Jacobian product computes $\epsilon^\top \nabla_{\vect{x}} \tilde{f}(\vect{x},t)$, and averaging a few probes per time step produces a low-variance estimate. To ensure consistency, the denoiser and schedule are evaluated with gradients enabled along the PF-ODE path, and COM projection is applied to \emph{both} drifts and probes.

\begin{algorithm}[h]
\caption{Training (embedding diffusion with cross-entropy)}
\label{alg:training}
\begin{algorithmic}[1]
\Require dataset $\{\vect{r}_0, y\}$, encoders $\mathrm{Enc}$, schedule $\sigma(t)$, EGNN denoiser $D$
\While{not converged}
  \State sample minibatch; form embeddings $\vect{e}_0 \leftarrow \mathrm{Enc}(y)$ (L2-normalize)
  \State COM-center coordinates; sample $t\sim p(t)$; set $\vect{x}_0 \gets [\vect{r}_0\|\vect{e}_0]$
  \State sample $\epsilon\sim\Normal(0,I)$, make coordinate block COM-free; set $\vect{x}_t \gets \vect{x}_0 + \sigma(t)\epsilon$
  \State $(\widehat{\vect{r}}_0,\vect{z}) \leftarrow D(\vect{x}_t, t)$; $\widehat{\vect{e}}_0 \leftarrow \sum_i \softmax(\vect{z})_i\,\vect{e}_i$
  \State compute $\mathcal{L}$ via \eqref{eq:loss}; update parameters
\EndWhile
\end{algorithmic}
\end{algorithm}

\begin{algorithm}[h]
\caption{PF-ODE likelihood per complex}
\label{alg:likelihood}
\begin{algorithmic}[1]
\Require clean state $\vect{x}_0$, schedule $\sigma$, denoiser $D$, time grid $0=t_0<\dots<t_T=1$
\State COM-center coordinates in $\vect{x}_0$; set $\ell \leftarrow 0$, $\vect{x}\leftarrow \vect{x}_0$
\For{$k=0,\dots,T-1$}
  \State $\widehat{\vect{x}}_0 \leftarrow \text{posterior mean from } D(\vect{x},t_k)$
  \State $\tilde{f} \leftarrow \alpha(t_k)\big(\vect{x}-\widehat{\vect{x}}_0\big)$ \hfill (COM-free)
  \State sample $m$ Hutchinson probes $\epsilon^{(j)}$; project coordinate blocks to COM-free
  \State $\widehat{\mathrm{div}} \leftarrow \tfrac{1}{m}\sum_j \epsilon^{(j)\top} \nabla_{\vect{x}} \tilde{f}(\vect{x},t_k)$
  \State $\ell \leftarrow \ell + \widehat{\mathrm{div}}\,(t_{k+1}{-}t_k)$
  \State Heun step: predictor $\vect{x}'\!=\!\vect{x}+\tilde{f}\,\Delta t$; corrector with $\tilde{f}'$ at $t_{k+1}$
  \State COM-center coordinates of the updated state
\EndFor
\State compute $\log p_1(\vect{x})$ with DOF correction; \Return $\log p_1(\vect{x}) + \ell$
\end{algorithmic}
\end{algorithm}

\subsubsection{Diffusion Model Training and Evaluation}

\paragraph{Data preparation}
To prepare the protein-ligand complex structures for diffusion model training and evaluation, all complexes underwent a standardized preprocessing. Structural files for proteins (PDB) and ligands (SDF) were parsed using Gemmi (v.0.7.0) and processed using Numpy (v.1.26.4), PyTorch (v.2.7.0) and PyTorch Geometric (v.2.6.1). 

Ligand atom coordinates and types were extracted from the Gemmi objects (hydrogens omitted). Ligand atom coordinates were combined into a ligand coordinate matrix ($\mathbf{C}_L$) and ligand atom types were one-hot-encoded to construct the ligand feature matrix ($\mathbf{F}_L$) using the following 10 categories:

\begin{itemize}[itemsep=1pt,topsep=5pt,parsep=0pt,partopsep=5pt]
    \item C, N, O, S, B, Br, Cl, P, I, F
\end{itemize}

For each residue in the protein, the pairwise distances between the residue's atoms and all ligand atoms were computed. A protein residue was designated as a pocket residue if any of its atoms were within 5Å of a ligand atom. The C-alpha ($\text{C}_\alpha$) coordinates of all identified pocket residues were combined to form the protein coordinate matrix ($\mathbf{C}_P$). The corresponding residue type 3-letter-codes were standardized and one-hot-encoded to create the protein feature matrix ($\mathbf{F}_P$), using the following 21 categories (with metal ions aggregated into a single group): 

\begin{itemize}[itemsep=1pt,topsep=5pt,parsep=0pt,partopsep=5pt]
    \item ALA, CYS, ASP, GLU, PHE, GLY, HIS, ILE, LYS, LEU, MET, ASN, PRO, GLN, ARG, SER, THR, VAL, TRP, TYR, METAL
\end{itemize}

In summary, four arrays were computed to represent a complex:

$$\begin{aligned}
\text{Ligand coordinates:} & \quad \mathbf{C}_L \in \mathbb{R}^{N_L \times 3} \\
\text{Protein coordinates:} & \quad \mathbf{C}_P \in \mathbb{R}^{N_P \times 3} \\
\text{Ligand features:} & \quad \mathbf{F}_L \in \mathbb{R}^{N_L \times 10} \\
\text{Protein features:} & \quad \mathbf{F}_P \in \mathbb{R}^{N_P \times 21}
\end{aligned}$$

where $N_L$ is the number of ligand atoms and $N_P$ is the number of pocket residues. These four arrays were then merged in a PyTorch geometric data object for each complex. To construct datasets, these data objects were combined into PyTorch datasets according to the splitting outlined in section \ref{met:Datasets}.

\paragraph{Training and hyperparameter optimization}
To maximize the model's performance and ensure robust training, we performed a systematic hyperparameter optimization. The search space was defined across several key architectural and training features, including learning rate, the number of sampling steps, the number of EGNN layers, the joint and hidden embedding sizes and the edge feature dimensionality. Initially, we employed a randomized search with short training runs to set reasonable search space boundaries. Then we explored the resulting search space outlined in Table~\ref{tab:hyperparams} using Bayesian optimization with the Optuna (v.4.4.0) python package. The best parameter combination was selected based on a joint assessment of validation loss and sampling-based evaluations, including percent of fragmented ligands, mean number of ligand fragments, mean number of rings, mean ring size and Jensen-Shannon divergences between training and sampled ligand atom and protein residue type distributions. The model was trained using the final configuration (in the rightmost column of the Table \ref{tab:hyperparams}), resulting in a model with with 6'386'128 parameters. The training process was continued for as long as the training and validation losses continued to decrease, but was stopped eventually after 1400 epochs. The last saved model (epoch 1390) was used for PF-ODE trajectory-based OOD evaluation. Protein-ligand interactions sampled from this model contained ligands with 92\% chemical validity, 72\% of unfragmented ligands and a mean number of 2.34 rings with mean size of 5.31 per molecule. Divergences in atom and amino acid type distributions (Jensen–Shannon) from the training data were very low, at 0.0017 and 0.004, respectively.

\begin{table}[htbp]
    \centering
    \small
    \renewcommand{\arraystretch}{1.2}
    \begin{tabular}{l p{8cm} c}
        \toprule
        \textbf{Parameter} & \textbf{Search Space} & \textbf{Selected} \\
        \midrule
        \texttt{num\_sampling\_steps} & $\{50, 100, 200, 300, 400\}$ & $400$ \\
        \texttt{joint\_nf} & $\{32, 64, 128, 256\}$ & $256$ \\
        \texttt{hidden\_nf} & $\{64, 128, 256\}$ & $256$ \\
        \texttt{n\_layers} & $\{4, 6, 8\}$ & $6$ \\
        \texttt{edge\_embedding\_dim} & $\{8, 16, 32, 64\}$ & $64$ \\
        \texttt{learning\_rate} & $\{5\times10^{-5}, 1\times10^{-4}, 5\times10^{-4}, 1\times10^{-3}, 5\times10^{-3}\}$ & $1\times10^{-4}$ \\
        \texttt{batch\_size} & $\{16, 32, 64, 128\}$ & $16$ \\
        \bottomrule
    \end{tabular}
    \caption{\textbf{Hyperparameter search spaces} and selected optimal values used for diffusion model training.}
    \label{tab:hyperparams}
\end{table}

\subsection{GEMS Training and Evaluation}
The Graph Neural Network (GNN) for efficient molecular scoring (GEMS) binding affinity prediction model \cite{GEMS} was trained using the same dataset splits as the diffusion model, following the authors' original instructions. 

\paragraph{Data preparation and training} Custom datasets were constructed for the training, validation, CASF2016, and OOD sets. This involved computing graph representations of the complexes across the entire PDBbind v.2020 dataset, featurized using language model embeddings derived from ESM2 \cite{Lin2022}, Ankh \cite{Elnaggar2023}, and ChemBERTa2 \cite{Ahmad2022}.
The GEMS model was trained on the training dataset using 5-fold cross-validation (CV) and the default hyperparameters: \begin{itemize}[itemsep=0.5pt]
    \item Loss Function: Root Mean Squared Error (RMSELoss)
    \item Optimizer: Stochastic Gradient Descent (SGD)
    \item Batch Size: 256
    \item Learning Rate: 0.001 (constant)
    \item Weight Decay: 0.001
\end{itemize}
The model was trained using early stopping, terminating the process when the validation RMSE did not decrease for 100 epochs. The last model that improved validation RMSE was then used.

\paragraph{Model selection and evaluation} We performed a hyperparameter search on the dropout rate applied before the final regression head. The model with a dropout rate of $0.6$ was selected, as it showed the lowest and most consistent validation loss across all five folds of the cross-validation. To evaluate performance on the validation, CASF2016, and OOD datasets, the predictions from the five models generated during the 5-fold CV were averaged to create an ensemble prediction. 

\paragraph{Evaluation metrics} The coefficient of determination ($R^2$) was chosen as the main readout for model performance. Other common regression metrics, such as the Pearson Correlation Coefficient (PCC) and Root Mean Squared Error (RMSE), were disregarded due to specific limitations. PCC is insensitive to scale offsets, meaning a high correlation score can coexist with large absolute errors. Conversely, prediction RMSE values are sensitive to variations in the label spread across different datasets, distorting performance comparisons. For datasets with a narrow label distribution, naive models (which predict the dataset mean for all complexes) can achieve deceptively low RMSE values, whereas the same naive models would achieve poor RMSE on datasets with a wide label distribution. $R^2$ normalizes performance relative to the target variance, providing a fairer comparison across diverse datasets.
To compute complex-wise errors, we applied a similar normalization to account for varying label spread across datasets. Specifically, we used Variance-Normalized Errors (VNE), calculated by dividing the absolute error of each prediction by the variance of the dataset's labels. This effectively rescales the error relative to the difficulty of the dataset, providing a metric that reflects how much the model improves sample scoring over a naive baseline that simply predicts the dataset mean.

\subsection{Protein-Ligand Complex Similarity Computation}
To compute a similarity score between two protein-ligand complexes, ligand similarity (Tanimoto), protein similarity (TM scores), and a pocket-aligned ligand root-mean-square-deviation (RMSD) were combined into an overall structure-based similarity score. This combination captures the presence of similar interaction patterns between two protein-ligand complexes on a structural level. These similarity scores were precomputed in our earlier work \cite{GEMS} and were reused here:

\begin{itemize}
    \item The Tanimoto similarity score to compare ligand structure was computed using RDKit library v.2024.03.3 based on count-based molecular fingerprints (GetCountFingerprint) of size 2048 and radius 2.
    
    \item TM-align \cite{Zhang2005} was used to assess protein structural similarity and to align complexes based on their protein residues. This algorithm identifies the best structural alignment between protein pairs, using a scoring function based on the RMSD of aligned residues and a length normalization factor, making it useful for identifying structural similarities regardless of sequence similarity. To find if one of the proteins is well represented within the other, we considered TM-scores normalized by the lengths of both amino acid chains and used the highest result.
    
    \item A root-mean-square-deviation (RMSD) is computed to compare ligand positioning in the binding pocket of the proteins. For this, the protein structures were aligned together with the bound ligands by applying the translation vector and rotation matrix that TM-align used to generate an optimal protein alignment. Then, the atom coordinates of the aligned ligands were compared by computing the RMSD between the nearest points in the two point clouds. In case of different atom counts, the distances from each atom in the larger ligand to its nearest neighbor in the smaller ligand, resulting in larger RMSD values for ligands with different atom counts.
\end{itemize}

To aggregate these similarity scores into an simple aggregate score representing similarity on the protein-ligand interaction level, we defined 
\[
  S = \max\left\{
    T_{\text{Tanimoto}} + T_{\text{TM-score}} + \bigl(1 - \mathrm{RMSD}\bigr),\; 0
  \right\}.
\]

Complex pairs possessing an identical pocket, identical ligands, and matching ligand binding pose achieve the maximum score of $3.0$. This score is primarily suited for our use case of detecting and comparing highly similar complexes and is robust in that range. However, its robustness decreases when comparing and ranking complexes in low-level similarity ranges, which is not relevant to our current analysis goals.

\subsection{Trajectory-aware out-of-distribution detection} \label{subsec:trajectory}
The PF-ODE in \eqref{eq:pf_ode_likelihood} yields a scalar log-likelihood $\ell(x)$ and, along the same forward integration, a set of trajectory summaries that quantify how the flow behaves on the way from clean to noise. We detect OOD samples by embedding these \emph{PF-ODE features} into a low-dimensional space and fitting separate density models for in-distribution (ID) and OOD trajectories. At test time we use a log-density ratio between the ID and OOD models as a trajectory-aware OOD score.

\paragraph{Feature construction}
We aggregate $d{=}19$ PF-ODE trajectory statistics per complex (full list in Fig.~\ref{fig:feature_importances}). For exposition we display only the most discriminative coordinates:
{\setlength{\arraycolsep}{2pt}
\[
\phi(x)=
\left[
\begin{array}{@{}l@{}}
\ell(x),\ \text{path tortuosity},\ \text{path efficiency},\\
\text{max Lipschitz},\ \text{VF spikiness},\ \text{mean acceleration},\\
\text{total flow energy},\ \text{VF }\ell_2\text{ mean},\ \text{coord--feature coupling, ...}
\end{array}
\right].
\]
}
In practice, we apply a train-only preprocessing pipeline (quantile transform, standardization, and PCA) to the full $19$-dimensional feature vector and retain the top $m$ principal components (with $m{=}15$ in our experiments).

\paragraph{Interpretation}
The selected statistics describe how hard the flow works and how steadily it does so. Paths that wander—summarized by tortuosity and its inverse, efficiency—indicate geometric inefficiency that is typical under shift. Local Lipschitz estimates read off stiffness: nearby states elicit disproportionately different drifts when the model is acting on OOD data. Field size and burstiness, captured by the vector-field $\ell_2$ mean and a spikiness ratio, reveal large corrective pushes that concentrate at intermediate noise scales. Mean acceleration records the lack of temporal smoothness in the drift, while total flow energy, $\sum_t f_t^\top \Delta x_t$, aggregates the directed work along the trajectory. Finally, coord–feature coupling measures how spatial updates co-move with feature updates across atoms and residues.

Empirically these quantities move together under distribution shift: stiffer and burstier fields, higher acceleration, and longer, less efficient paths tend to co-occur, whereas in-distribution trajectories remain shorter, smoother, and energetically modest. This coordinated variation explains their strong loadings on the leading PCA component, yielding a single, high-variance axis along which ID and OOD separate cleanly.

\paragraph{Class-conditional density model}
Let $\mathcal{I}$ index ID training samples and $\mathcal{O}$ index OOD calibration samples for a given shift condition, with corresponding PF-ODE features
\[
X_{\mathrm{ID}} = \{\phi(x_i)\}_{i\in\mathcal{I}},
\qquad
X_{\mathrm{OOD}} = \{\phi(x_j)\}_{j\in\mathcal{O}}.
\]
We first learn a preprocessing map
\[
T \;=\; \Pi_m \circ Z \circ Q,
\]
where $Q$ is a marginal quantile transform to normal scores (fitted on $X_{\mathrm{ID}}$), $Z$ is standardization (mean/variance estimated from $X_{\mathrm{ID}}$), and $\Pi_m$ projects onto the top $m$ principal components (PCA fit on $X_{\mathrm{ID}}$). This yields low-dimensional embeddings
\[
z_i = T\phi(x_i)\in\mathbb{R}^m, \quad i\in\mathcal{I},
\qquad
z_j = T\phi(x_j)\in\mathbb{R}^m, \quad j\in\mathcal{O}.
\]

In this PC space we fit separate Gaussian KDEs for ID and OOD:
\[
\hat p_{\mathrm{ID}}(z)
=
\frac{1}{|X_{\mathrm{ID}}|\,h_{\mathrm{ID}}^m}
\sum_{i\in\mathcal{I}}
\frac{1}{(2\pi)^{m/2}}
\exp\!\Big(-\tfrac{1}{2}\big\|\tfrac{z-z_i}{h_{\mathrm{ID}}}\big\|_2^2\Big),
\]
\[
\hat p_{\mathrm{OOD}}(z)
=
\frac{1}{|X_{\mathrm{OOD}}|\,h_{\mathrm{OOD}}^m}
\sum_{j\in\mathcal{O}}
\frac{1}{(2\pi)^{m/2}}
\exp\!\Big(-\tfrac{1}{2}\big\|\tfrac{z-z_j}{h_{\mathrm{OOD}}}\big\|_2^2\Big),
\]
with isotropic bandwidths $h_{\mathrm{ID}}, h_{\mathrm{OOD}}>0$. Each bandwidth is selected by $K$-fold cross-validation on the corresponding class, maximizing the mean log-likelihood on held-out folds.

\paragraph{Scoring and decision rule (test time)}
For a test complex $x^\star$ we compute $z^\star = T\phi(x^\star)$ and define the log-density ratio
\begin{equation}
\label{eq:traj-ldr-score}
L(x^\star)
\;=\;
\log \hat p_{\mathrm{ID}}\!\big(z^\star\big)
\;-\;
\log \hat p_{\mathrm{OOD}}\!\big(z^\star\big).
\end{equation}
Large positive values of $L(x^\star)$ indicate that the trajectory looks more ID-like than OOD-like, while negative values indicate the opposite. We use the negated ratio
\begin{equation}
\label{eq:traj-ood-score}
S(x^\star) \;=\; -\,L(x^\star)
\end{equation}
as a continuous OOD score, where larger $S$ means more OOD-like. A binary decision is obtained by thresholding $S(x)$ at a level $\tau$ chosen on a labeled calibration set of ID and OOD examples, for example by sweeping $\tau$ to optimize a desired trade-off between false positives and true positives. AUROC is reported by sweeping $\tau$ over the full range of observed scores, treating OOD as the positive class. Note that
\[
L(x) \;=\; \log\frac{P_{\mathrm{ID}}(x)}{P_{\mathrm{OOD}}(x)},
\]
so $L(x)>0$ favors the ID hypothesis and $L(x)<0$ favors the OOD hypothesis. All results reported below are obtained without any additional outlier trimming of either class.

\paragraph{Bootstrap evaluation under class imbalance}
For each OOD split we partition ID and OOD complexes into training, validation, and test subsets. The two KDEs and the preprocessing map $T$ are fit on their respective training data, and the decision threshold $\tau$ is fixed once by maximizing the F1-score for the OOD class on the validation set. Because the number of ID decoys can be much larger than the number of OOD binders, naive evaluation on the full test pool would overweight specificity. To obtain a balanced and robust estimate, we keep the OOD test set fixed and repeatedly draw balanced ID test subsets of the same size. For each draw we evaluate the classifier, treating OOD as the positive class, and record AUROC, accuracy, precision, recall, F1 and specificity. We finally report, for each dataset, the mean and standard deviation of these metrics over multiple such bootstrap test subsets. All plots use a representative test subset, while the legends and tables display the aggregated mean $\pm$ standard deviation.

\paragraph{Confidence and risk stratification}
We obtain a notion of confidence by ranking $S(x)$ against the empirical score distribution on ID calibration data. The percentile rank of $S(x)$ within this reference distribution can be mapped to qualitative risk levels (e.g.\ Low/Medium/High) via fixed quantile cutoffs, providing a simple way to stratify predictions by trajectory atypicality.

\paragraph{Why trajectory features help}
The scalar PF-ODE likelihood $\ell(x)$ already aggregates information across noise scales via the divergence integral in \eqref{eq:pf_ode_likelihood}. Trajectory summaries add complementary signals that are characteristic of OOD behavior: larger and spikier vector fields, higher and less stable local Lipschitz factors, increased acceleration, and longer, less efficient paths between states. Modeling the \emph{joint} behavior of these quantities (including $\ell(x)$) through a class-conditional density ratio in a low-dimensional PC space yields substantially stronger separation than calibrating $\ell(x)$ alone, while remaining simple and data-efficient. This ensemble of features is what allowed us to successfully draw a strong decision boundary between the training data and the $\alpha$-carbonic anhydrase structures (\texttt{3dd0}) dataset where the likelihood scores did not provide sufficient discriminative signal. We provide further analysis of this special case in Section~\ref{suppl:robustness}.

\paragraph{Empirical performance}
Table \ref{tab:traj-ldr-ood} summarizes the performance of the trajectory-aware LDR classifier across multiple OOD test splits. For each dataset we report the mean $\pm$ standard deviation of AUROC, accuracy, F1, specificity, precision and recall over balanced bootstrap test subsets, with OOD treated as the positive class. Overall, the method achieves high AUROC and strong specificity on most shifts, with some degradation on the most challenging benchmarks.

\subsection{Embedding–space baseline on GEMS}
\label{sec:baseline-gems}

\paragraph{Methodology}
We construct a non–generative baseline that operates directly in the representation space of a pretrained GEMS encoder. For each complex \(x\), we run the encoder \(R\) stochastic passes (e.g.\ dropout, data augmentation) and obtain vectors \(\{E_r(x)\in\mathbb{R}^d\}_{r=1}^{R}\). We then store a single summary embedding
\[
\bar{E}(x)\;:=\;\frac{1}{R}\sum_{r=1}^{R} E_r(x)\in\mathbb{R}^d,
\]
which serves as a fixed descriptor of \(x\) in all subsequent experiments. All baselines described below operate only on these precomputed embeddings and never modify or retrain GEMS.

\paragraph{Unsupervised detectors on ID embeddings}
Given the ID training set embeddings \(\{\bar{E}(x)\}_{x\in\mathcal{D}_{\mathrm{ID}}^{\mathrm{train}}}\), we fit a family of classical one–class / anomaly detectors. Before fitting, we standardize all features with a single affine transform (mean and variance estimated on \(\mathcal{D}_{\mathrm{ID}}^{\mathrm{train}}\)). Each detector then assigns an anomaly score \(S(x)\) to a test embedding, with the convention that larger values indicate evidence for OOD:
\begin{enumerate}
  \item \textbf{Mahalanobis distance.} We estimate an ID mean \(\mu\in\mathbb{R}^d\) and a regularized covariance
  \[
  \Sigma \;=\; \mathrm{Cov}\bigl(\bar{E}(x)\bigr) \;+\; \lambda I_d,\qquad \lambda>0,
  \]
  and score a point by its Mahalanobis distance
  \[
  S_{\mathrm{mah}}(x)\;=\;\sqrt{\bigl(\bar{E}(x)-\mu\bigr)^{\!\top}\Sigma^{-1}\bigl(\bar{E}(x)-\mu\bigr)}.
  \]
  \item \textbf{k–NN distance.} We build a \(k\)–nearest–neighbor index over the ID embeddings (Euclidean metric in the standardized space). The score is the mean distance to the \(k\) nearest ID neighbors,
  \[
  S_{\mathrm{knn}}(x)\;=\;\frac{1}{k}\sum_{i=1}^{k}\left\|\bar{E}(x)-\bar{E}(x^{(i)})\right\|_2,
  \]
  where \(x^{(i)}\) denotes the \(i\)-th nearest ID training complex to \(x\).
  \item \textbf{Local Outlier Factor (LOF).} We fit a Local Outlier Factor model on the ID embeddings with novelty detection enabled. LOF returns a negative log–density–like score; we negate it so that larger values correspond to more anomalous points:
  \[
  S_{\mathrm{lof}}(x)\;=\; - \mathrm{LOF\_score}(x).
  \]
  \item \textbf{Isolation Forest.} We fit an Isolation Forest on the ID embeddings. The model outputs higher scores for typical points; we again negate these values,
  \[
  S_{\mathrm{if}}(x)\;=\; - \mathrm{IF\_score}(x),
  \]
  so that higher scores indicate higher isolation (OOD).
  \item \textbf{One–Class SVM.} Finally, we fit a one–class SVM with an RBF kernel on the ID embeddings. Its decision function is positive for inliers and negative for outliers; we define
  \[
  S_{\mathrm{svm}}(x)\;=\; - \mathrm{OCSVM\_decision}(x),
  \]
  so that larger values again correspond to stronger OOD evidence.
\end{enumerate}
These five detectors cover distance–based, density–ratio, isolation–based and margin–based paradigms, all applied to the same fixed GEMS representation.

\paragraph{Calibration and decision rules}
For each OOD test split, we re–use the same ID training embeddings but construct a small validation set to obtain a decision threshold. Concretely, given ID and OOD embeddings for a particular split, we randomly partition both into validation and test subsets. Each detector \(S\) is trained on \(\mathcal{D}_{\mathrm{ID}}^{\mathrm{train}}\) only and then evaluated on the validation embeddings:
\[
\{S(x)\}_{x\in\mathcal{D}_{\mathrm{ID}}^{\mathrm{val}}}, \qquad
\{S(x)\}_{x\in\mathcal{D}_{\mathrm{OOD}}^{\mathrm{val}}}.
\]
We treat OOD as the positive class and search over a grid of candidate thresholds \(\tau\), selecting the one that maximizes the F1 score on this validation set. This yields a single calibrated threshold \(\tau_S\) per detector and per OOD split, which is \emph{not} adjusted on the test data.

\paragraph{Evaluation protocol}
On the test embeddings for a given split, we report both threshold–free and thresholded performance. For AUROC, we use the raw anomaly scores \(S(x)\) and treat OOD as the positive class. For calibrated decisions, we predict
\[
\hat{y}(x)\;=\;\mathbb{I}\bigl\{ S(x) > \tau_S \bigr\}
\]
and compute accuracy, precision, recall, F1 and specificity. To reduce variance and to match the protocol used for the trajectory–aware LDR classifier, all metrics are estimated on balanced bootstrap test subsets: in each of \(B\) bootstrap rounds we subsample the ID test set to match the size of the OOD test set, evaluate the metrics, and finally report the mean and standard deviation over these \(B\) replicates.

We evaluate performance across the same held-out OOD splits used for the trajectory-aware LDR classifier. For each split we train all five detectors (Mahalanobis, $k$-NN distance, Local Outlier Factor, Isolation Forest, and one-class SVM) on the ID training embeddings, and then \emph{select the best detector per dataset by AUROC}. This gives a generous reference for what can be achieved by post-hoc anomaly detection on a fixed GEMS representation.

\paragraph{Why we emphasize AUROC}
In this setting AUROC is a more robust indicator of detector quality than the thresholded metrics. First, AUROC is threshold-free: it depends only on the ranking induced by the anomaly scores and is invariant to the particular decision threshold. In contrast, accuracy, F1, precision, recall and specificity are all evaluated at a \emph{single} threshold that we explicitly tune on a small validation set to maximize F1. This calibration step is useful to obtain a reasonable operating point, but it can inflate F1 and accuracy even for relatively weak detectors, especially on the balanced 50/50 ID/OOD mixtures used in our bootstrap protocol. Secondly, the balanced evaluation itself is convenient for comparison but does not reflect realistic class imbalance where OOD events are rare; under such imbalance, the same threshold would typically yield very different precision and accuracy, while AUROC remains unchanged. For these reasons, when comparing to the trajectory-aware LDR classifier we primarily interpret AUROC as the main measure of intrinsic separability between ID and OOD, and view F1/accuracy as secondary, operating-point-dependent statistics.

\paragraph{Rationale and limitations}
This embedding–space baseline assumes that ID complexes occupy a relatively compact region of the GEMS representation space, and that OOD complexes are either geometrically distant from this region (captured by \(S_{\mathrm{mah}}\) and \(S_{\mathrm{knn}}\)) or sparsely supported / difficult to fit by standard one–class models (captured by LOF, Isolation Forest and one–class SVM). The approach is entirely post–hoc: it requires no generative training, no access to the PF–ODE dynamics, and can be implemented with off–the–shelf anomaly detection packages given a frozen GEMS encoder. At the same time, it is a purely static snapshot of representation geometry and does not integrate information along noise–time trajectories or exploit intermediate superpositions of identities.

\subsection{Rate-In baseline on GEMS}

We include Rate-In \cite{zeevi2025rate} as a non-diffusion baseline for OOD detection. Rate-In builds on MC dropout~\cite{gal2016dropout}, but replaces a fixed global dropout rate with \emph{adaptive, per-sample, per-layer} dropout rates chosen to preserve information flow through the network. Intuitively, in-distribution (ID) inputs admit relatively large dropout rates without disrupting the computation, whereas OOD inputs are much more fragile under dropout perturbations. Rate-In turns this robustness gap into an OOD signal.

\paragraph{Information-preserving dropout}
Consider a network with $L$ dropout layers and layer-$\ell$ activations $h_{\text{in}}^{(\ell)}$ and $h_{\text{out}}^{(\ell)}$. For a given input $x$ and dropout rate $p_\ell$ at layer $\ell$, Rate-In measures how much information about $h_{\text{in}}^{(\ell)}$ survives through dropout by approximating the mutual information
\[
I_\ell(p_\ell) \approx I\big(h_{\text{in}}^{(\ell)}; h_{\text{out}}^{(\ell)} \,\big|\, p_\ell\big).
\]
The “no-dropout” reference
\[
I_{\text{full}}^{(\ell)} := I\big(h_{\text{in}}^{(\ell)}; h_{\text{out}}^{(\ell)} \,\big|\, p_\ell = 0\big)
\]
is computed once per input. The relative information loss at layer $\ell$ is
\[
\delta I_\ell(p_\ell) = \frac{I_{\text{full}}^{(\ell)} - I_\ell(p_\ell)}{I_{\text{full}}^{(\ell)}}.
\]
Rate-In chooses the largest dropout rate $p_\ell^\star$ for which $\delta I_\ell(p_\ell)$ stays below a target threshold $\epsilon$ (we use $\epsilon = 0.1$ in all experiments).

We approximate
\[
I(X;Y) \approx -\tfrac{1}{2}\log(1 - \rho^2),
\]
where $\rho$ is the Pearson correlation between flattened $X$ and $Y$. To reduce variance, we average over $K=5$ forward passes with independent dropout masks at the current $p_\ell$.

\paragraph{Per-layer rate optimisation}
For each input $x$ and each dropout layer $\ell$, we solve for $p_\ell^\star$ with a simple multiplicative update:

\begin{enumerate}
    \item Compute $I_{\text{full}}^{(\ell)}$ with dropout disabled at layer $\ell$.
    \item Initialise $p_\ell$ (we use $p_\ell = 0.3$).
    \item Repeat for up to $N_{\max}=20$ iterations:
    \begin{enumerate}
        \item Enable dropout only at layer $\ell$ with rate $p_\ell$, run $K$ stochastic forward passes, and estimate $I_\ell(p_\ell)$ by averaging the MI over these passes.
        \item Compute $\delta I_\ell(p_\ell)$.
        \item If $\lvert \delta I_\ell(p_\ell) - \epsilon\rvert$ is below a tolerance $\delta$ (we use $\delta=0.02$), stop and set $p_\ell^\star \gets p_\ell$.
        \item If $\delta I_\ell(p_\ell) > \epsilon$ (too much information loss), decrease dropout: $p_\ell \leftarrow 0.9\,p_\ell$.
        \item If $\delta I_\ell(p_\ell) < \epsilon$ (too little information loss), increase dropout: $p_\ell \leftarrow \min(1.1\,p_\ell, 0.9)$.
    \end{enumerate}
\end{enumerate}

This yields, for each input, an optimised rate vector $\mathbf{p}^\star = (p_1^\star,\dots,p_L^\star)$ and per-layer iteration counts $n_\ell$.

\paragraph{OOD scoring from variance and optimisation dynamics}
Given the optimised rates $\mathbf{p}^\star$, we run MC dropout with all layers active and fixed rates $p_\ell^\star$, producing $M=30$ stochastic predictions $\{\hat y_m\}_{m=1}^M$ for the regression output (binding affinity in our case).

Rate-In then constructs a \emph{scalar} OOD score by combining three signals:

\begin{enumerate}
    \item \textbf{Predictive variance}
    \[
    \sigma_{\text{pred}}^2 = \mathrm{Var}\big(\{\hat y_m\}_{m=1}^M\big),
    \]
    the usual MC-dropout epistemic uncertainty.
    \item \textbf{Rate variance}
    \[
    \sigma_{\text{rate}}^2 = \mathrm{Var}\big(\{p_\ell^\star\}_{\ell=1}^L\big),
    \]
    which measures how heterogeneous the optimised dropout rates are across layers.
    \item \textbf{Convergence difficulty}
    \[
    \bar n = \frac{1}{L}\sum_{\ell=1}^L n_\ell,
    \]
    the average number of iterations required to satisfy the information-loss constraint.
\end{enumerate}

We combine these into a single OOD score
\begin{equation}
s_{\text{OOD}}(x)
= \sigma_{\text{pred}}^2
+ \lambda_1 \,\sigma_{\text{rate}}^2
+ \lambda_2 \,\frac{\bar n}{N_{\max}},
\label{eq:ratein-ood-score}
\end{equation}
with fixed weights $\lambda_1 = 0.3$ and $\lambda_2 = 0.2$. Large values of $s_{\text{OOD}}$ indicate high uncertainty and/or unstable information flow, and are interpreted as evidence for OOD.

\paragraph{Instantiation on GEMS}
We apply Rate-In to the GEMS graph neural network~\cite{GEMS}, a protein–ligand binding affinity predictor. GEMS operates on molecular graphs with rich node, edge, and global features and contains $L=6$ dropout layers across its node/edge MLPs and global pooling blocks.

To implement Rate-In:

\begin{itemize}
    \item We register PyTorch forward hooks on every dropout-containing module to cache $h_{\text{in}}^{(\ell)}$ and $h_{\text{out}}^{(\ell)}$ for MI estimation.
    \item For batched molecular graphs, activations are flattened over nodes/edges within each molecule before computing correlations.
    \item During rate optimisation for a given layer $\ell$, we disable dropout in all other layers and iteratively adjust $p_\ell$ as described above.
    \item After obtaining $\mathbf{p}^\star$, we enable dropout in all layers and collect $M$ predictions to compute $s_{\text{OOD}}(x)$ via~\eqref{eq:ratein-ood-score}.
\end{itemize}




\subsection{Author Contributions}

VA developed the diffusion methodology and software stack, implementing the full pipeline from scratch for molecular complex dynamics: (i) training infrastructure (losses, conditioning, and sampling for evaluation), (ii) diffusion/noise schedules and time discretization, (iii) numerical solvers and sampling procedures, (iv) probability-flow ODE (PF-ODE) integration, and (v) likelihood evaluation via PF-ODE divergence estimation (including Hutchinson trace estimation). VA ported and integrated a modified EGNN backbone adapted from DiffSBDD into this pipeline. VA further conceived and implemented the CDCD-style unified continuous diffusion scheme to embed discrete chemical identities and continuous coordinates in a shared continuous state, enabling a single joint diffusion process and self-consistent likelihood/OOD evaluation. VA additionally designed and implemented the PF-ODE trajectory-statistics scoring features, trained the trajectory-scoring models using trajectory statistics computed by DG, implemented and trained both baselines on GEMS outputs generated by DG, and wrote the majority of the methods and the mathematical sections.

\medskip DG curated the datasets, defined the data splitting strategies, and formulated the underlying graph modeling approach. DG conducted the training of the diffusion models, implementing a custom evaluation pipeline to monitor diffusion model learning and biochemical quality metrics of sampled structures simultaneously throughout the training process. DG executed the large-scale computation of log-likelihoods and trajectory features across all datasets using a pipeline developed by VA. DG further performed the bioinformatic validation of the diffusion-based OOD scores and trained and evaluated GEMS for binding affinity predictions. Finally, DG performed the data analysis, designed the figures, and wrote the majority of the main text.\\
\\
SM conceived the project and DG initiated the collaboration \\
\\
RB provided advice throughout the project and critically reviewed the manuscript. \\
\\
The project was supervised by SM

\subsection{Acknowledgments}
We thank Dr. Peter Stockinger for carefully reading the manuscript and providing valuable feedback.

\clearpage
\bibliography{diffusion_paper}

\clearpage

\setcounter{section}{0}
\renewcommand{\thesection}{S\arabic{section}}
\renewcommand{\thesubsection}{S\arabic{section}.\arabic{subsection}}
\setcounter{figure}{0}
\renewcommand{\thefigure}{S\arabic{figure}}
\setcounter{table}{0}
\renewcommand{\thetable}{S\arabic{table}}

\section{Supplementary Information}

\subsection{Related Work}
\label{sec:related_work}
The problem of detecting out-of-distribution (OOD) samples has been extensively studied in computer vision. Early discriminative approaches relied on post-hoc analysis of supervised classifier outputs, such as the maximum softmax probability \cite{Hendrycks2017} or temperature scaling \cite{Liang2017}. More recent methods leverage energy-based scores \cite{Liu2020} or activation pruning \cite{Sun2021} to separate ID from OOD data. However, these methods usually require training labels of the in-distribution data and their applicability is thus limited to labeled datasets.

\paragraph{Deep generative models for OOD detection} 
Generative models are particularly well suited for OOD detection in an unsupervised setting, as they are trained on the generative task and require no labels. During training, these models learn a probability density function of the training data. This allows to estimate the likelihood of a given new sample under the learned model \cite{Song2020, Goodier2023}. 

Some generative approaches for OOD detection are reconstruction-based, predominantly using Variational Autoencoders (VAEs) \cite{AnCho2015, Yang2021, Ataeiasad2024}. The underlying assumption is that a model trained on in-distribution data will fail to accurately reconstruct OOD samples, yielding a high reconstruction error. However, some research highlighted that standard autoencoders often reconstruct outliers perfectly, rendering the reconstruction error unreliable \cite{Denouden2018, Zong2018}. 

To overcome this, research shifted toward explicit density estimation using the likelihood $p(x)$ directly to identify OOD samples. These methods model the in-distribution data with a probabilistic model and flag test data that lies in low-density regions as OOD. While theoretically sound, this direction faced a major setback when \cite{Nalisnick2018} demonstrated the likelihood paradox, where generative models assigned higher likelihoods to simple OOD data (e.g., SVHN) than to the complex training distribution (e.g., CIFAR-10). This was confirmed by following research, showing that likelihood exhibits a strong bias towards the input complexity of the data \cite{Serra2019, Kirichenko2020}. Some works have attempted to mitigate this using likelihood ratios \cite{Ren2019, Goodier2023}, complexity-based corrections \cite{Serra2019} or applying diffusion in the representation space of a pre-trained image encoder \cite{Ding2025}. 

Deviating from pure likelihood-based methods, Heng et al. (2024) \cite{Heng2024} introduced diffusion trajectory features, specifically the rate-of-change and curvature of the PF-ODE, as effective OOD signals for image benchmarks. Their trajectory-focused approach proved resistant to the complexity bias that often confounds likelihood-based OOD classification. However, while their study was limited to image data and a small set of descriptors, we translate these principles to the challenging domain of complex 3D graphs. We expand the feature space to 18 distinct trajectory features and model them jointly with log-likelihoods. This composite approach yields a robust OOD detector for 3D graph data, where likelihoods and trajectory features provide complementary, essential signals for classification in this domain.

\paragraph{OOD detection on graphs} 
Extending OOD detection from image data to graph domains introduces unique challenges due to the irregular, non-Euclidean nature of graph data. While several works have adapted discriminative scores for Graph Neural Networks (GNNs) trained with standard supervised objectives \cite{Wu2023}, unsupervised generative approaches typically focus on reconstructing the adjacency matrix conditioned on node features ($P(A|X)$). These methods operate on the assumption that a failure to recreate a graph's structure indicates an outlier \cite{Li2022, Zhang2025}. However, existing unsupervised graph-level OOD detectors \cite{Ma2021, Li2022, Shen2024} generally operate on 2D topological graphs, neglecting the precise 3D geometry that is essential for molecular complexes. Shen et al. (2024) \cite{Shen2024} use a diffusion model for OOD detection on molecular graphs of drug-like compounds. However, they also limit their modeling to topological graphs without 3D geometry, rendering their method unsuitable for our use case of OOD detection in protein-ligand interactions. In molecular interaction graphs, OOD characteristics manifest not only in features or connectivity but crucially in geometry. Since the distinction between in-distribution and OOD samples often hinges on 3D spatial properties, such as bond lengths or steric clashes, methods that do not account for the precise spatial positioning of nodes would fail to detect "geometric outliers" that are topologically valid but physically impossible. Our work addresses this by combining graph features and 3D geometry into a single representation for the diffusion process.

\paragraph{Diffusion models for protein-ligand interactions} 
Score-based generative models and Denoising Diffusion Probabilistic Models (DDPMs) have achieved state-of-the-art results generating new protein-ligand interactions \cite{Hoogeboom2022,TargetDiff,Schneuing2024}. However, these works primarily focused on the generative task of sampling novel, stable conformations. The potential of using the probability flow ODE formulation \cite{Song2020} of a diffusion model for OOD detection in 3D molecular complexes remains unexplored. Our work bridges this gap by adapting the equivariant diffusion backbone from DiffSBDD \cite{Schneuing2024} for precise OOD detection, thereby providing a much-needed OOD detection capability to the growing field of machine learning for protein-ligand interactions.

\subsection{Error-log-likelihood relation}
\begin{figure}[h!]
    \centering
    \includegraphics[width=\textwidth]{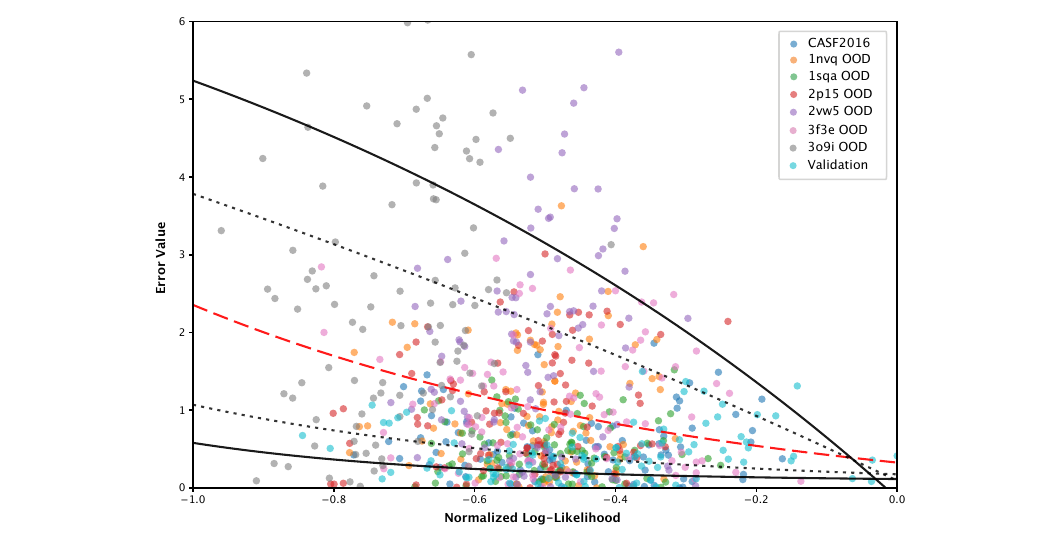}
    \caption{\textbf{Error-log-likelihood relation follows a scaled and shifted exponential:}
    Scatterplot illustrating the relationship between the log-likelihoods assigned by the diffusion model (x-axis) and the corresponding GEMS binding affinity prediction errors (y-axis) across the validation, CASF2016, and six out-of-distribution (OOD) datasets. A hierarchical fitting of five scaled and shifted exponential curves shows that log-likelihoods provide a rough estimate of the expected GEMS prediction error.  Across all non-training complexes ($N=6223$), 75.8\% had predicted errors that fell within the bounds of the fitted exponential curves, with 8.84\% high outliers (errors above the fit) and 15.39\% low outliers (errors below the fit). To allow fair comparison across datasets with different label spreads, absolute errors are normalized by the label variance, preventing artificially low error values on narrowly distributed datasets. Exponential curves are fit to a maximum of 500 randomly sampled points from each non-training distribution. Only 100 randomly sampled complexes from each datasets are shown in the to improve visibility.
    }
    \label{fig:lkhd_error_exponential}
\end{figure}

\subsection{Evaluation of sensitivity to ligand size and chemical features} \label{suppl:robustness}
Naturally, the size of the input ligand and pocket influence the PF-ODE log-likelihood scores. As anticipated, datasets with significant divergence in ligand and pocket size distributions relative to the training data generally received lower scores, indicating OOD status. However, the diffusion model demonstrated a notable robustness to this size-induced distribution shift. Many complexes containing extraordinarily large ligands were nonetheless assigned high log-likelihoods, indicating that size is not the sole driver of the OOD signal. This resilience is exemplified by the 3f3e dataset: Despite possessing by far the largest ligand and pocket sizes (Figure \ref{fig:pocket_ligand_sizes}), it yielded intermediate log-likelihoods. This suggests an intermediate OOD status, a conclusion which was confirmed by our bioinformatic OOD quantification (Figure \ref{fig:lkhd_tm_tanimoto_rmsd}), which identified 3f3e to present a moderate distribution shift relative to the training data. This suggests that the diffusion model recognizes familiar structural motifs and assigns high likelihoods, even when the input's size strongly deviates from the common sizes in the training distribution.

\begin{figure}[h]
    \centering
    \includegraphics[width=\textwidth]{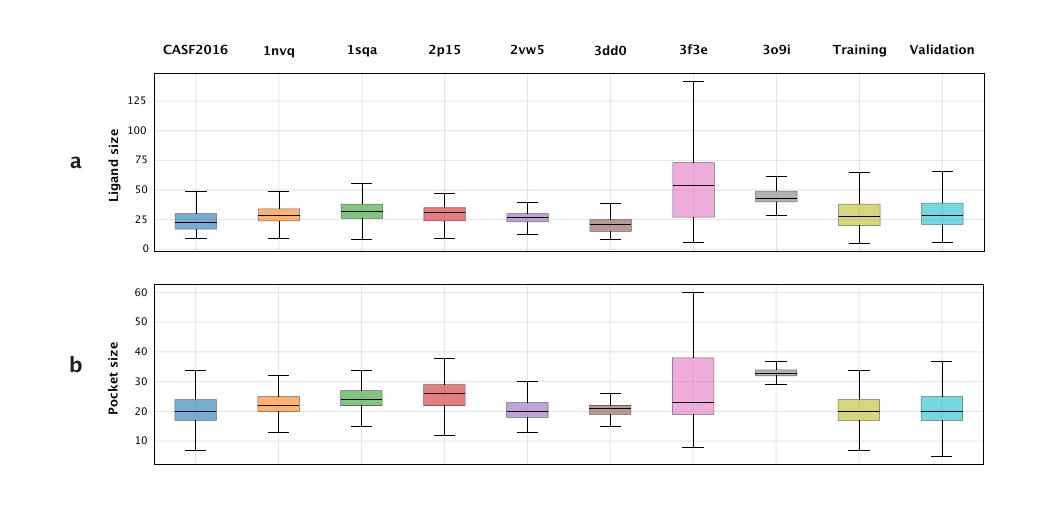}
    \caption{\textbf{Ligand and pocket size distribution across datasets:}
    This figure compares the distributions of \textbf{a)} ligand sizes (atom counts) and \textbf{b)} protein pocket sizes (amino acid counts) across all datasets, including CASF2016 ($N=285$), 1nvq ($N=2684$), 1sqa ($N=734$), 2p15 ($N=455$), 2vw5 ($N=207$), 3dd0 ($N=461$), 3f3e ($N=389$), 3o9i ($N=468$), training ($N=10'168$) and validation ($N=1'127$) datasets. As pocket residues are determined using a fixed radius around each ligand atom, ligand and pocket sizes are strongly connected. 
    Boxplots show the median (centre line), 25th–75th percentiles (box), whiskers extend to data points within 1.5 × IQR, outliers are not shown.
    }
    \label{fig:pocket_ligand_sizes}
\end{figure}

\subsection{Intra-Cluster Ligand Diversity}
While we isolated specific protein families into OOD datasets, the ligand space remained uncontrolled. Consequently, OOD pockets bind a diverse mix of ligands, ranging from familiar to exotic structures. Figure \ref{fig:intra_ligand_diversity} shows the intra-cluster ligand diversity of each dataset. For every dataset, we identified a medoid ligand, defined as the molecule with the smallest sum of distances ($1 - \text{Tanimoto similarity}$) to all other ligands, and computed the distance of all other ligands to this central reference. 

\begin{figure}[h!]
    \centering
    \includegraphics[width=0.8\textwidth]{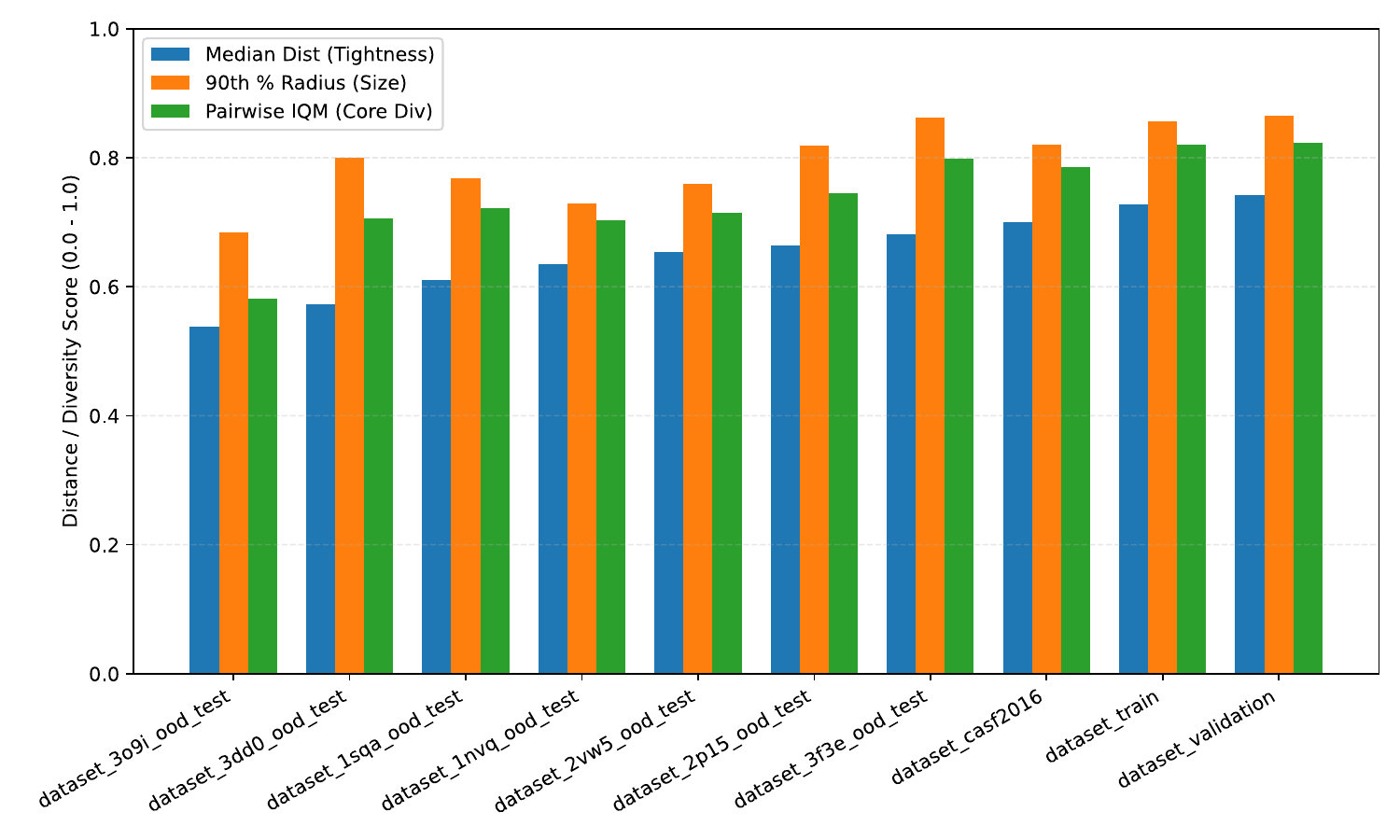}
    \caption{\textbf{Intra-cluster ligand diversity of all datasets} sorted by cluster tightness in ascending order. Three robust intra-cluster diversity metrics are included: 1) Median distance to medoid measures typical cluster tightness by computing the median distance from the most representative point to all other points. 2) 90th percentile radius captures robust cluster size by identifying the distance that encompasses 90\% of points from the medoid. 3) Pairwise inter-quartile mean (IQM) quantifies overall cluster cohesion by averaging distances between all unique pairs within the middle 50\% of the pairwise distance distribution, providing resistance to outliers.}
    \label{fig:intra_ligand_diversity}
\end{figure}

\newpage
\subsection{Sampling Evaluation}\label{sec:sampling_eval}
The distribution-learning performance of the diffusion model was tested by quality checks on sampled graphs during model training. The validity of the \textit{de novo} generated protein–ligand configurations was assessed using the following metrics:
\begin{itemize}[itemsep=1pt]
    \item Divergence of atom type distribution from training distribution
    \item Divergence of amino acid type distribution from training distribution
    \item Percentage of fragmented ligands and mean number of fragments after bond inference
    \item Chemical validity (RDKit sanitization success)
    \item Mean number of rings and mean ring size per molecule
\end{itemize}
The best model achieved 92\% chemical validity, 84\% unfragmented ligands, and an average of 2.72 rings per ligand with a mean ring size of 5.59. Divergences in atom and amino acid type distributions (Jensen–Shannon) from the training data were low, at 0.010 and 0.009, respectively. These results indicate that the model learned the training distribution well enough to effectively map Gaussian noise to realistic 3D atom placements, resulting in chemically valid \textit{de novo} ligands.

\subsection{PF-ODE likelihoods as predictors of predictive error}
\label{subsec:pf-ode-error-theory}

We formalize the connection between the probability--flow ODE (PF-ODE)
log-likelihoods assigned by the diffusion model and the prediction errors of
an independent binding-affinity model such as GEMS.
Throughout, $x$ denotes a protein-ligand complex represented as a graph with
3D coordinates and discrete atom/residue types
(Section~\ref{sec:diffusion_model}).

\paragraph{Two-model setting}
Let $p_0$ be the (unknown) data distribution on complexes, assumed to admit
a density $p_0(x)$ with respect to Lebesgue measure on the ambient
Euclidean space (COM-centered coordinates concatenated with continuous
embeddings).
We train:
\begin{itemize}
  \item a \emph{diffusion model} with PF-ODE density $q_\phi(x)$ and score
  field $s_\phi(x,t)$, parameterised by $\phi$, and

  \item an \emph{independent predictor} (e.g.\ GEMS) with parameters
  $\theta$, which for each input $x$ outputs a scalar prediction
  $\hat{y}_\theta(x)$ of the true binding affinity $y(x)$.
\end{itemize}
We define the pointwise prediction error as
\begin{equation}
  e_\theta(x) \;:=\; \bigl|\hat{y}_\theta(x) - y(x)\bigr|,
\end{equation}
or its variance-normalised version used in our experiments
(Section~\ref{subsec:gems_correlations}).

The diffusion model yields a per-complex log-likelihood via the PF-ODE
(Section~\ref{subsec:pf-ode}); we write
\begin{equation}
  \ell_\phi(x) \;:=\; \log q_\phi(x),
  \qquad
  L_\phi(x) \;:=\; -\ell_\phi(x) \;=\; -\log q_\phi(x),
\end{equation}
for the log-likelihood and the negative log-likelihood (NLL),
respectively.  In what follows we omit the subscript $\phi$ when
unambiguous and write $q,q_t,s$ instead of $q_\phi,q_t^\phi,s_\phi$.

Our goal in this section is to relate $L(x)$, which is fully determined by the
diffusion model, to the error $e_\theta(x)$ of the \emph{separate}
predictive model, under explicit and mild assumptions.

\subsubsection{Relative entropy along the PF-ODE}

We recall the PF-ODE setup from Section~\ref{sec:diffusion_model}.
Let $\{p_t\}_{t\in[0,1]}$ and $\{q_t\}_{t\in[0,1]}$ be the time-marginal
densities of the noising process and the learned PF-ODE, respectively, so that
\begin{equation}
  p_t(x) \;=\; p(x_t = x),
  \qquad
  q_t(x) \;=\; q_\phi(x_t = x).
\end{equation}
The true and learned PF-ODE drifts are denoted by
$v^\star(x,t)$ and $v_\phi(x,t)$, and both families satisfy the
continuity equations
\begin{equation}
  \partial_t p_t(x) \;=\; -\nabla_x \!\cdot \bigl(p_t(x)\,v^\star(x,t)\bigr),
  \qquad
  \partial_t q_t(x) \;=\; -\nabla_x \!\cdot \bigl(q_t(x)\,v_\phi(x,t)\bigr).
\end{equation}

\begin{assumption}[Regularity of the PF-ODE marginals]
\label{ass:regularity}
For every $t\in[0,1]$ the densities $p_t$ and $q_t$ are smooth and
strictly positive on $\R^d$, and satisfy sufficient decay at infinity
so that:
\begin{itemize}
  \item differentiation under the integral sign is justified;
  \item integration by parts over $\R^d$ produces no boundary terms.
\end{itemize}
\end{assumption}

\emph{Interpretation.}
This is a standard smoothness and decay assumption for continuity
equations and PF-ODEs.  In our setting, $p_1$ and $q_1$ are equal to
the same COM-constrained Gaussian (Section~\ref{subsec:pf-ode}), and
the PF-ODE preserves smoothness when started from smooth data.

\paragraph{KL divergence and relative Fisher information}
For each $t\in[0,1]$ we define the relative entropy (KL divergence)
and relative Fisher information between $p_t$ and $q_t$ as
\begin{align}
  D_t
  &\;:=\;
  \mathrm{KL}\bigl(p_t \,\|\, q_t\bigr)
  \;=\;
  \int_{\R^d}
    p_t(x)\,\log\!\frac{p_t(x)}{q_t(x)}\,\dd x,
  \label{eq:Dt-def}
  \\
  I_t
  &\;:=\;
  \int_{\R^d}
    p_t(x)\,\bigl\|
      \nabla_x \log p_t(x) - \nabla_x \log q_t(x)
    \bigr\|_2^2 \,\dd x.
  \label{eq:It-def}
\end{align}
In our molecular setting, $x$ contains all COM-centered coordinates and
continuous embeddings of atoms and residues, stacked into a single Euclidean
vector.

\begin{lemma}[Time derivative of KL along the PF-ODE]
\label{lem:dt-kl-exact}
Under Assumption~\ref{ass:regularity},
\begin{equation}
  \frac{\dd}{\dd t}D_t
  \;=\;
  \int_{\R^d}
    p_t(x)\,
    \bigl(v^\star(x,t) - v_\phi(x,t)\bigr)^\top
    \Bigl[
      \nabla_x \log p_t(x) - \nabla_x \log q_t(x)
    \Bigr]\,
    \dd x.
  \label{eq:dt-kl-exact}
\end{equation}
\end{lemma}

\begin{proof}
Starting from the definition \eqref{eq:Dt-def},
\begin{align}
  \frac{\dd}{\dd t}D_t
  &=
  \frac{\dd}{\dd t}
  \int_{\R^d} p_t(x)\,\log\!\frac{p_t(x)}{q_t(x)}\,\dd x
  \\
  &=
  \int_{\R^d} (\partial_t p_t(x))\,
        \log\!\frac{p_t(x)}{q_t(x)}\,\dd x
  \;+\;
  \int_{\R^d} p_t(x)\,\partial_t\log\!\frac{p_t(x)}{q_t(x)}\,\dd x.
  \label{eq:dt-kl-split}
\end{align}
For the second term we expand
\begin{equation}
  \partial_t\log\!\frac{p_t}{q_t}
  =
  \frac{\partial_t p_t}{p_t} - \frac{\partial_t q_t}{q_t},
\end{equation}
so that
\begin{align}
  \int p_t\,\partial_t\log\!\frac{p_t}{q_t}\,\dd x
  &=
  \int \partial_t p_t(x)\,\dd x
  \;-\;
  \int \frac{p_t(x)}{q_t(x)}\,\partial_t q_t(x)\,\dd x.
\end{align}
Because $p_t$ is a probability density, $\int \partial_t p_t\,\dd x =
\partial_t \int p_t\,\dd x = 0$, hence
\begin{equation}
  \frac{\dd}{\dd t}D_t
  =
  \int (\partial_t p_t)\,\log\!\frac{p_t}{q_t}\,\dd x
  \;-\;
  \int \frac{p_t}{q_t}\,\partial_t q_t\,\dd x.
  \label{eq:dt-kl-before-ce}
\end{equation}
Using the continuity equations
\begin{equation}
  \partial_t p_t = -\nabla_x\cdot(p_t v^\star),
  \qquad
  \partial_t q_t = -\nabla_x\cdot(q_t v_\phi),
\end{equation}
we obtain
\begin{align}
  \frac{\dd}{\dd t}D_t
  &=
  -\int \nabla_x\cdot(p_t v^\star)\,\log\!\frac{p_t}{q_t}\,\dd x
  \;+\;
  \int \frac{p_t}{q_t}\,\nabla_x\cdot(q_t v_\phi)\,\dd x.
\end{align}
Assumption~\ref{ass:regularity} allows us to integrate by parts without
boundary terms.  For the first term,
\begin{align}
  -\int \nabla_x\cdot(p_t v^\star)\,\log\!\frac{p_t}{q_t}
  \,\dd x
  &=
  \int p_t(x)\,v^\star(x,t)\cdot\nabla_x\log\!\frac{p_t(x)}{q_t(x)}\,\dd x.
\end{align}
For the second term,
\begin{align}
  \int \frac{p_t}{q_t}\,\nabla_x\cdot(q_t v_\phi)\,\dd x
  &=
  -\int q_t(x)\,v_\phi(x,t)\cdot
  \nabla_x\!\left(\frac{p_t(x)}{q_t(x)}\right)\,\dd x
  \\
  &=
  -\int p_t(x)\,v_\phi(x,t)\cdot
  \nabla_x\log\!\frac{p_t(x)}{q_t(x)}\,\dd x.
\end{align}
Combining these identities gives
\begin{equation}
  \frac{\dd}{\dd t}D_t
  =
  \int p_t(x)\,
    \bigl(v^\star(x,t) - v_\phi(x,t)\bigr)\cdot
    \nabla_x\log\!\frac{p_t(x)}{q_t(x)}\,\dd x,
\end{equation}
which is \eqref{eq:dt-kl-exact}.
\end{proof}

\emph{Interpretation.}
Lemma~\ref{lem:dt-kl-exact} is an exact information-balance identity:
the rate of change of the KL divergence between the true and learned
PF-ODE marginals is the inner product, averaged under $p_t$, between
(i) the velocity mismatch $v^\star - v_\phi$ and
(ii) the gradient of the log-density ratio
$\log\frac{p_t}{q_t}$.

\begin{lemma}[Cauchy--Schwarz bound on $\dot D_t$]
\label{lem:dt-kl-cs}
Under the assumptions above,
\begin{equation}
  \bigl|\dot D_t\bigr|
  \;\le\;
  \left(
    \int_{\R^d}
      p_t(x)\,\bigl\|v^\star(x,t) - v_\phi(x,t)\bigr\|_2^2\,
      \dd x
  \right)^{\!1/2}
  \sqrt{I_t},
  \label{eq:dt-kl-cs}
\end{equation}
where $I_t$ is the relative Fisher information from
\eqref{eq:It-def}.
\end{lemma}

\begin{proof}
Apply the Cauchy--Schwarz inequality to
\eqref{eq:dt-kl-exact}.
\end{proof}

\emph{Interpretation.}
The instantaneous rate of change of KL is bounded by the product of two
factors:
(i) the root-mean-squared velocity error between the true and learned
PF-ODEs, and (ii) the square root of the Fisher information of the
log-density ratio.

\paragraph{Relation to score error}
In our VE setting, the PF-ODE drifts and scores are linked by
\begin{equation}
  v^\star(x,t) = -\tfrac{1}{2} g(t)^2\,s^\star(x,t),
  \qquad
  v_\phi(x,t) = -\tfrac{1}{2} g(t)^2\,s_\phi(x,t),
\end{equation}
with scalar diffusion coefficient $g(t)$.
Hence
\begin{equation}
  v^\star(x,t) - v_\phi(x,t)
  \;=\;
  -\tfrac{1}{2} g(t)^2\,
  \bigl(s^\star(x,t) - s_\phi(x,t)\bigr).
\end{equation}
Define
\begin{equation}
  A_t
  \;:=\;
  \int_{\R^d}
    p_t(x)\,\bigl\|v^\star(x,t) - v_\phi(x,t)\bigr\|_2^2\,
    \dd x.
\end{equation}
Then
\begin{equation}
  A_t
  =
  \frac{g(t)^4}{4}
  \int_{\R^d}
    p_t(x)\,\bigl\|s^\star(x,t) - s_\phi(x,t)\bigr\|_2^2\,
    \dd x.
\end{equation}

We quantify the time-integrated score error as
\begin{equation}
  \mathcal{L}_{\mathrm{score}}(\phi)
  \;:=\;
  \int_0^1
    g(t)^4
    \int_{\R^d}
      p_t(x)\,\bigl\|s^\star(x,t) - s_\phi(x,t)\bigr\|_2^2\,
      \dd x\,\dd t,
  \label{eq:l-score-def}
\end{equation}
which differs from the practical training loss only by a known
time-dependent weight (Section~\ref{sec:diffusion_model}).
With this definition,
\begin{equation}
  \int_0^1 A_t\,\dd t
  \;=\;
  \frac{1}{4}\,\mathcal{L}_{\mathrm{score}}(\phi).
\end{equation}

Integrating \eqref{eq:dt-kl-cs} in time from $t=0$ to $t=1$ and using
that $D_1 = \mathrm{KL}(p_1\|q_1) = 0$ (common terminal Gaussian
under COM constraints), we obtain
we obtain
\begin{equation}
  D_0
  \;=\;
  -\int_0^1 \dot D_t\,\dd t
  \;\le\;
  \int_0^1 \bigl|\dot D_t\bigr|\,\dd t
  \;\le\;
  \int_0^1 \sqrt{A_t I_t}\,\dd t.
\end{equation}

Another application of Cauchy--Schwarz in time yields
\begin{equation}
  D_0
  \;\le\;
  \left(\int_0^1 A_t\,\dd t \right)^{1/2}
  \left(\int_0^1 I_t\,\dd t \right)^{1/2}
  \;=\;
  \frac{1}{2}\,
  \sqrt{\mathcal{L}_{\mathrm{score}}(\phi)}
  \left(\int_0^1 I_t\,\dd t \right)^{1/2}.
  \label{eq:kl-score-fisher}
\end{equation}

\begin{assumption}[Bounded integrated relative Fisher]
\label{ass:bounded-fisher}
There exists a constant $C_{\mathrm{FI}}>0$ such that
\begin{equation}
  \int_0^1 I_t\,\dd t
  \;\le\;
  C_{\mathrm{FI}}.
\end{equation}
\end{assumption}

\emph{Interpretation.}
Assumption~\ref{ass:bounded-fisher} states that the time-integrated
relative Fisher information between the true and learned marginals
remains bounded.  In other words, along the entire PF-ODE trajectory
the score fields $\nabla_x \log p_t$ and $\nabla_x \log q_t$ do not
develop arbitrarily large discrepancies in gradient norm.
When $q_t=p_t$ for all $t$ (perfect modelling), we have $I_t\equiv 0$,
so the assumption trivially holds.

Combining \eqref{eq:kl-score-fisher} with
Assumption~\ref{ass:bounded-fisher} yields:

\begin{proposition}[Score error controls final-time KL up to capacity]
\label{prop:score-kl}
Under Assumptions~\ref{ass:regularity} and~\ref{ass:bounded-fisher},
\begin{equation}
  \mathrm{KL}\bigl(p_0\|q_0\bigr)
  \;=\;
  D_0
  \;\le\;
  \frac{1}{2}\,\sqrt{C_{\mathrm{FI}}}\,
  \sqrt{\mathcal{L}_{\mathrm{score}}(\phi)}.
\end{equation}
\end{proposition}

\emph{Interpretation.}
Proposition~\ref{prop:score-kl} states that if the PF-ODE velocity
field learned by the diffusion model tracks the true probability-flow
velocity field in mean-squared sense (small $\mathcal{L}_{\mathrm{score}}$),
and if the relative Fisher information along the flow stays bounded,
then the resulting density $q_0$ is close to the true data distribution
$p_0$ in KL.  In our context, this means that a well-trained molecular
diffusion model assigns a globally reasonable density to the ensemble
of protein-ligand complexes seen during training.

\subsubsection{PF-ODE negative log-likelihood as a random variable}

We now treat the PF-ODE negative log-likelihood
\begin{equation}
  L(x) \;:=\; -\log q_0(x)
\end{equation}
as a random variable under the data distribution $p_0$.

The Shannon entropy of $p_0$ is
\begin{equation}
  H(p_0)
  \;:=\;
  -\int_{\R^d} p_0(x)\,\log p_0(x)\,\dd x.
\end{equation}
By the definition of KL divergence,
\begin{align}
  \mathrm{KL}(p_0\|q_0)
  &=
  \int_{\R^d}
    p_0(x)\,\log\!\frac{p_0(x)}{q_0(x)}\,\dd x
  \\
  &=
  \int_{\R^d}
    p_0(x)\,\log p_0(x)\,\dd x
  - \int_{\R^d}
    p_0(x)\,\log q_0(x)\,\dd x
  \\
  &=
  -H(p_0)
  + \int_{\R^d}
      p_0(x)\,L(x)\,\dd x.
\end{align}
Rearranging gives the exact identity
\begin{equation}
  \int_{\R^d}
    p_0(x)\,L(x)\,\dd x
  \;=\;
  H(p_0) + \mathrm{KL}(p_0\|q_0).
\end{equation}
We denote this typical (mean) NLL by
\begin{equation}
  L_{\mathrm{typ}}
  \;:=\;
  \int_{\R^d}
    p_0(x)\,L(x)\,\dd x
  \;=\;
  H(p_0) + \mathrm{KL}(p_0\|q_0).
\end{equation}

\emph{Interpretation.}
The mean PF-ODE NLL under the true data distribution decomposes into
an intrinsic term (the entropy $H(p_0)$) and a mismatch term
($\mathrm{KL}(p_0\|q_0)$).
Proposition~\ref{prop:score-kl} implies that
$\mathrm{KL}(p_0\|q_0)$ is small when the diffusion model is well-trained,
so $L_{\mathrm{typ}}$ is close to the intrinsic complexity of $p_0$.

\begin{assumption}[Finite NLL variance]
\label{ass:nll-variance}
There exists a constant $\sigma^2<\infty$ such that
\begin{equation}
  \mathrm{Var}_{x\sim p_0}\bigl(L(x)\bigr)
  \;=\;
  \int_{\R^d}
    p_0(x)\,\bigl(L(x) - L_{\mathrm{typ}}\bigr)^2\,\dd x
  \;\le\;
  \sigma^2.
\end{equation}
\end{assumption}

\emph{Interpretation.}
Assumption~\ref{ass:nll-variance} requires that the PF-ODE NLL does
not exhibit infinite variance under the training distribution.
For any sane trained density model this holds in practice and imposes
a minimal regularity condition on the tails of $q_0$ with respect to
$p_0$.

Under this assumption, a basic Chebyshev bound yields tail control:

\begin{lemma}[NLL concentration under $p_0$]
\label{lem:nll-tail-chebyshev}
Under Assumption~\ref{ass:nll-variance}, for any $\alpha>0$,
\begin{equation}
  \int_{\{x:\,L(x)\,\ge\,L_{\mathrm{typ}}+\alpha\}}
    p_0(x)\,\dd x
  \;\le\;
  \frac{\sigma^2}{\alpha^2}.
  \label{eq:nll-tail}
\end{equation}
\end{lemma}

\begin{proof}
Apply Chebyshev's inequality to the centered random variable
$Z(x) := L(x) - L_{\mathrm{typ}}$:
\begin{equation}
  \int_{\{x:\,|Z(x)|\,\ge\,\alpha\}} p_0(x)\,\dd x
  \;\le\;
  \frac{\mathrm{Var}_{p_0}(Z)}{\alpha^2}
  \;\le\;
  \frac{\sigma^2}{\alpha^2}.
\end{equation}
Since the event $\{L(x)\ge L_{\mathrm{typ}}+\alpha\}$ is contained in
$\{|Z(x)|\ge\alpha\}$, the same bound holds for
\eqref{eq:nll-tail}.
\end{proof}

Lemma~\ref{lem:nll-tail-chebyshev} provides a minimal concentration
statement: most training complexes have PF-ODE NLLs within a band of
width $O(\alpha)$ around the typical value $L_{\mathrm{typ}}$, with a
tail probability decaying at least as $1/\alpha^2$.

\subsubsection{Coupling PF-ODE likelihoods to GEMS errors}

We now link the PF-ODE NLL $L(x)$ of the diffusion model to the
prediction error $e_\theta(x)$ of an independent model such as GEMS.
In our experiments (Section~\ref{subsec:gems_correlations}), we observe
a strong empirical relationship between PF-ODE log-likelihoods and
GEMS binding-affinity errors: complexes with low likelihood
(high $L(x)$) tend to exhibit large prediction errors, whereas
high-likelihood complexes have much smaller errors on average.

We capture this relationship by a monotone error envelope:

\begin{assumption}[Monotone NLL-error envelope]
\label{ass:monotone-envelope}
There exists a non-decreasing function
$\phi:\R\to[0,\infty)$ such that, for $p_0$-almost every
complex $x$,
\begin{equation}
  e_\theta(x)
  \;\le\;
  \phi\bigl(L(x)\bigr).
\end{equation}
\end{assumption}

Assumption~\ref{ass:monotone-envelope} states that for any two
complexes $x_1,x_2$ with $L(x_1)\le L(x_2)$ (i.e.\ $x_1$ is at least as
likely as $x_2$ under the diffusion model), the worst-case prediction
error at $x_1$ is no larger than the worst-case error at $x_2$.
The function $\phi$ acts as a \emph{calibration curve} mapping PF-ODE
NLL values to an upper bound on the predictive error.
In practice, $\phi$ can be obtained empirically by fitting a
non-decreasing upper envelope (e.g.\ a scaled and shifted exponential)
to the scatter plot of $(L(x), e_\theta(x))$ on a held-out calibration
set (Figure~\ref{fig:lkhd_error_exponential}). This is in accordance with the further empirical findings explained in Section~\ref{subsec:gems_correlations}; for visualizations, refer to Figures~\ref{fig:lkhd_R2_boxplot} and~\ref{fig:lkhd_error_heatmaps} in the main text.

\begin{theorem}[High-probability error bound from PF-ODE NLL]
\label{thm:likelihood-error}
Suppose Assumptions~\ref{ass:regularity},
\ref{ass:bounded-fisher}, \ref{ass:nll-variance},
and~\ref{ass:monotone-envelope} hold, and that the VE relation
between velocities and scores in Section~\ref{sec:diffusion_model}
is satisfied.
Let $L_{\mathrm{typ}}$ denote the typical NLL under $p_0$, and fix any
margin $\alpha>0$. Then
\begin{equation}
  \int_{\{x:\,e_\theta(x)\,>\,\phi(L_{\mathrm{typ}}+\alpha)\}}
    p_0(x)\,\dd x
  \;\le\;
  \frac{\sigma^2}{\alpha^2}.
  \label{eq:likelihood-error-bound}
\end{equation}
Equivalently,
\begin{equation}
  \mathbb{P}_{x\sim p_0}\Bigl(
    e_\theta(x)
    \;\le\;
    \phi\bigl(L_{\mathrm{typ}}+\alpha\bigr)
  \Bigr)
  \;\ge\;
  1 - \frac{\sigma^2}{\alpha^2}.
\end{equation}
\end{theorem}

\begin{proof}
Consider the event
\begin{equation}
  A_\alpha
  \;:=\;
  \{x:\,L(x) \le L_{\mathrm{typ}} + \alpha\}.
\end{equation}
On $A_\alpha$, Assumption~\ref{ass:monotone-envelope} and monotonicity 
of $\phi$ yield
\begin{equation}
  e_\theta(x)
  \;\le\;
  \phi\bigl(L(x)\bigr)
  \;\le\;
  \phi\bigl(L_{\mathrm{typ}}+\alpha\bigr)
  \qquad\text{for all }x\in A_\alpha.
\end{equation}
Therefore the event
$\{x:\,e_\theta(x) > \phi(L_{\mathrm{typ}}+\alpha)\}$ is contained in
the complement $A_\alpha^c$, and
\begin{equation}
  \int_{\{x:\,e_\theta(x)\,>\,\phi(L_{\mathrm{typ}}+\alpha)\}}
    p_0(x)\,\dd x
  \;\le\;
  \int_{A_\alpha^c} p_0(x)\,\dd x
  =
  \int_{\{x:\,L(x)\,>\,L_{\mathrm{typ}}+\alpha\}} p_0(x)\,\dd x.
\end{equation}
Applying Lemma~\ref{lem:nll-tail-chebyshev} to the right-hand side,
\begin{equation}
  \int_{\{x:\,L(x)\,>\,L_{\mathrm{typ}}+\alpha\}} p_0(x)\,\dd x
  \;\le\;
  \frac{\sigma^2}{\alpha^2},
\end{equation}
which establishes \eqref{eq:likelihood-error-bound}.
\end{proof}

Theorem~\ref{thm:likelihood-error} makes the following statement in
our concrete setting:
\begin{itemize}
  \item The diffusion model is trained on the same structural
  distribution of protein-ligand complexes as GEMS, but using only
  3D geometry and discrete types (no affinities).
  \item If the PF-ODE is well trained (small score loss
  $\mathcal{L}_{\mathrm{score}}(\phi)$) and its density $q_0$ does not
  have pathological NLL variance, then there exists a typical NLL level
  $L_{\mathrm{typ}}$ such that:
  \begin{quote}
    On almost all in-distribution complexes
    (all but a fraction $O(\sigma^2/\alpha^2)$),
    a near-typical PF-ODE NLL
    $L(x)\le L_{\mathrm{typ}}+\alpha$ implies a uniform upper bound
    on the GEMS prediction error:
    $e_\theta(x)\le\phi(L_{\mathrm{typ}}+\alpha)$.
  \end{quote}
  \item Conversely, complexes with unusually high PF-ODE NLL
  $L(x)\gg L_{\mathrm{typ}}$ fall outside this guarantee.  For such
  complexes the theorem does not enforce any small error bound, which is
  consistent with our empirical observation that low-likelihood
  complexes exhibit heavy-tailed GEMS errors
  (Figures~\ref{fig:lkhd_error_heatmaps} and
   \ref{fig:lkhd_error_exponential}).
\end{itemize}
From a practical perspective, PF-ODE log-likelihoods provide a
task-agnostic, generative notion of ``typicality'' for protein-ligand
complexes.  Theorem~\ref{thm:likelihood-error} explains why these
likelihoods can serve as a useful proxy for the reliability of a
separate predictive model such as GEMS: up to smoothness and
concentration assumptions, being high-likelihood under the diffusion
model is a sufficient condition for having small prediction error with
high probability under the training distribution.

\subsection{Comprehensive Error Tables}
To facilitate a rigorous comparison across our diffusion trajectory-based OOD detection and the implemented baselines Rate-In and Embedding Space, we present detailed performance metrics for each approach. These comprehensive error tables provide an overview over the performance on the OOD datasets employed in this work, reporting the mean and standard deviation for AUROC, accuracy, F1-score, specificity, precision, and recall. By evaluating our diffusion trajectory approach alongside the Embedding space and Rate-In baselines, we can assess not only the raw discriminative power (AUROC) but also the practical reliability and classification balance of each method. All results are derived using the balanced bootstrap evaluation strategy to ensure that the reported statistics are resilient to the inherent class imbalances present in the protein-family-based OOD splits.
\begin{samepage}

\begin{table}[h]
\centering
\footnotesize
\begin{tabular}{lcccccc}
\toprule
Dataset & AUROC & Acc & F1 & Spec & Prec & Recall \\
\midrule
\texttt{casf2016}  & $0.554\,{\pm}\,0.019$ & $0.558\,{\pm}\,0.016$ & $0.559\,{\pm}\,0.009$ & $0.557\,{\pm}\,0.032$ & $0.559\,{\pm}\,0.018$ & $0.559\,{\pm}\,0.000$\\
\texttt{1nvq}      & $0.804\,{\pm}\,0.007$ & $0.785\,{\pm}\,0.005$ & $0.803\,{\pm}\,0.004$ & $0.694\,{\pm}\,0.009$ & $0.741\,{\pm}\,0.006$ & $0.876\,{\pm}\,0.000$\\
\texttt{1sqa}      & $0.928\,{\pm}\,0.008$ & $0.830\,{\pm}\,0.007$ & $0.818\,{\pm}\,0.006$ & $0.899\,{\pm}\,0.013$ & $0.883\,{\pm}\,0.013$ & $0.762\,{\pm}\,0.000$\\
\texttt{2p15}      & $0.721\,{\pm}\,0.017$ & $0.585\,{\pm}\,0.010$ & $0.473\,{\pm}\,0.006$ & $0.797\,{\pm}\,0.020$ & $0.649\,{\pm}\,0.022$ & $0.373\,{\pm}\,0.000$\\
\texttt{2vw5}      & $0.659\,{\pm}\,0.033$ & $0.715\,{\pm}\,0.020$ & $0.753\,{\pm}\,0.013$ & $0.564\,{\pm}\,0.039$ & $0.666\,{\pm}\,0.021$ & $0.866\,{\pm}\,0.000$\\
\texttt{3dd0}      & $0.728\,{\pm}\,0.019$ & $0.740\,{\pm}\,0.012$ & $0.764\,{\pm}\,0.008$ & $0.636\,{\pm}\,0.024$ & $0.699\,{\pm}\,0.014$ & $0.844\,{\pm}\,0.000$\\
\texttt{3f3e}      & $0.919\,{\pm}\,0.009$ & $0.824\,{\pm}\,0.007$ & $0.804\,{\pm}\,0.007$ & $0.926\,{\pm}\,0.015$ & $0.907\,{\pm}\,0.017$ & $0.722\,{\pm}\,0.000$\\
\texttt{3o9i}      & $0.922\,{\pm}\,0.010$ & $0.866\,{\pm}\,0.008$ & $0.862\,{\pm}\,0.007$ & $0.895\,{\pm}\,0.016$ & $0.888\,{\pm}\,0.015$ & $0.837\,{\pm}\,0.000$\\
\bottomrule
\end{tabular}
\scriptsize
\caption{\textbf{Embedding space GEMS baseline OOD detection performance across held-out splits:} For each dataset, we report the mean $\pm$ standard deviation of AUROC, accuracy (Acc), F1, and specificity (Spec) over balanced bootstrap test subsets, treating OOD as the positive class. For each split we take the \emph{best} detector by AUROC among Mahalanobis, $k$-NN distance, Local Outlier Factor, Isolation Forest and one-class SVM. The final rows report the mean and standard deviation across datasets for these summary metrics. Per-dataset precision (Prec) and recall of the embedding-space GEMS baseline, using the best detector per dataset selected by AUROC. We report mean $\pm$ standard deviation over balanced bootstrap test subsets, with OOD as the positive class.
}
\label{tab:gems-ood}
\end{table}

\begin{table}[h]
\centering
\footnotesize
\begin{tabular}{lcccccc}
\toprule
Dataset & AUROC & Acc & F1 & Spec & Prec & Recall \\
\midrule
\texttt{casf2016  } & $0.444\,{\pm}\,0.041$   & $0.492\,{\pm}\,0.027$   & $0.605\,{\pm}\,0.100$   & $0.169\,{\pm}\,0.154$   & $0.484\,{\pm}\,0.072$        & $0.815\,{\pm}\,0.172$\\
\texttt{1nvq      } & $0.581\,{\pm}\,0.013$   & $0.584\,{\pm}\,0.009$   & $0.645\,{\pm}\,0.016$   & $0.411\,{\pm}\,0.045$   & $0.562\,{\pm}\,0.008$        & $0.757\,{\pm}\,0.045$\\
\texttt{1sqa      } & $0.664\,{\pm}\,0.024$   & $0.642\,{\pm}\,0.021$   & $0.667\,{\pm}\,0.021$   & $0.565\,{\pm}\,0.050$   & $0.624\,{\pm}\,0.021$        & $0.718\,{\pm}\,0.042$\\
\texttt{2p15      } & $0.541\,{\pm}\,0.039$   & $0.545\,{\pm}\,0.030$   & $0.616\,{\pm}\,0.042$   & $0.350\,{\pm}\,0.116$   & $0.534\,{\pm}\,0.024$        & $0.740\,{\pm}\,0.111$\\
\texttt{2vw5      } & $0.439\,{\pm}\,0.047$   & $0.487\,{\pm}\,0.024$   & $0.586\,{\pm}\,0.135$   & $0.180\,{\pm}\,0.221$   & $0.470\,{\pm}\,0.098$        & $0.795\,{\pm}\,0.233$\\
\texttt{3dd0      } & $0.337\,{\pm}\,0.029$   & $0.498\,{\pm}\,0.005$   & $0.238\,{\pm}\,0.317$   & $0.641\,{\pm}\,0.478$   & $0.179\,{\pm}\,0.239$        & $0.354\,{\pm}\,0.472$\\
\texttt{3f3e      } & $0.741\,{\pm}\,0.032$   & $0.690\,{\pm}\,0.030$   & $0.649\,{\pm}\,0.034$   & $0.806\,{\pm}\,0.054$   & $0.750\,{\pm}\,0.047$        & $0.574\,{\pm}\,0.046$\\
\texttt{3o9i      } & $0.859\,{\pm}\,0.024$   & $0.828\,{\pm}\,0.022$   & $0.838\,{\pm}\,0.021$   & $0.764\,{\pm}\,0.038$   & $0.791\,{\pm}\,0.027$        & $0.892\,{\pm}\,0.029$\\
\texttt{validation} & $0.503\,{\pm}\,0.022$   & $0.501\,{\pm}\,0.015$   & $0.533\,{\pm}\,0.131$   & $0.365\,{\pm}\,0.250$   & $0.495\,{\pm}\,0.053$        & $0.638\,{\pm}\,0.254$\\

\bottomrule
\end{tabular}
\scriptsize
\caption{\textbf{Rate-In OOD detection performance across held-out splits:} For each OOD dataset, we report the mean $\pm$ standard deviation of AUROC, accuracy (Acc), F1, and specificity (Spec) over 100 bootstrap iterations with 70/30 train/test splits, treating OOD as the positive class. The final rows report the mean and standard deviation across datasets for these summary metrics. For each OOD dataset, we report precision and recall over 100 bootstrap iterations.}
\label{tab:ratein-ood}
\end{table}

\begin{table}[h]
\centering
\footnotesize
\begin{tabular}{lcccccc}
\toprule
Dataset & AUROC & Acc & F1 & Spec & Prec & Recall \\
\midrule
\texttt{casf2016}  & $0.708\,{\pm}\,0.042$ & $0.604\,{\pm}\,0.029$ & $0.449\,{\pm}\,0.018$ & $0.887\,{\pm}\,0.058$ & $0.753\,{\pm}\,0.100$ & $0.321\,{\pm}\,0.000$\\
\texttt{1nvq}      & $0.842\,{\pm}\,0.013$ & $0.744\,{\pm}\,0.011$ & $0.736\,{\pm}\,0.009$ & $0.776\,{\pm}\,0.023$ & $0.761\,{\pm}\,0.018$ & $0.712\,{\pm}\,0.000$\\
\texttt{1sqa}      & $0.888\,{\pm}\,0.013$ & $0.787\,{\pm}\,0.014$ & $0.758\,{\pm}\,0.012$ & $0.906\,{\pm}\,0.027$ & $0.878\,{\pm}\,0.031$ & $0.667\,{\pm}\,0.000$\\
\texttt{2p15}      & $0.754\,{\pm}\,0.034$ & $0.693\,{\pm}\,0.027$ & $0.680\,{\pm}\,0.019$ & $0.734\,{\pm}\,0.054$ & $0.713\,{\pm}\,0.043$ & $0.651\,{\pm}\,0.000$\\
\texttt{2vw5}      & $0.939\,{\pm}\,0.033$ & $0.888\,{\pm}\,0.028$ & $0.886\,{\pm}\,0.025$ & $0.912\,{\pm}\,0.057$ & $0.911\,{\pm}\,0.053$ & $0.864\,{\pm}\,0.000$\\
\texttt{3dd0}      & $0.951\,{\pm}\,0.009$ & $0.956\,{\pm}\,0.009$ & $0.955\,{\pm}\,0.009$ & $0.976\,{\pm}\,0.019$ & $0.976\,{\pm}\,0.019$ & $0.936\,{\pm}\,0.000$\\
\texttt{3f3e}      & $0.639\,{\pm}\,0.061$ & $0.598\,{\pm}\,0.035$ & $0.429\,{\pm}\,0.021$ & $0.897\,{\pm}\,0.070$ & $0.765\,{\pm}\,0.125$ & $0.300\,{\pm}\,0.000$\\
\texttt{3o9i}      & $0.963\,{\pm}\,0.006$ & $0.949\,{\pm}\,0.011$ & $0.948\,{\pm}\,0.011$ & $0.986\,{\pm}\,0.022$ & $0.985\,{\pm}\,0.023$ & $0.913\,{\pm}\,0.000$\\

\bottomrule
\end{tabular}
\scriptsize
\caption{
\textbf{Performance of the diffusion trajectory LDR OOD detector across held-out splits:} For each OOD dataset, we report the mean $\pm$ standard deviation of AUROC, accuracy (Acc), F1, specificity (Spec), Precision (Prec) and recall over balanced bootstrap test subsets, treating OOD as the positive class.}
\label{tab:traj-ldr-ood}
\end{table}
\end{samepage}


\end{document}